\documentclass[table,xcdraw,12pt]{article}

\newcommand{\shorthead}{The LLM Policy Cliff}
\usepackage{aiarticle}

\usepackage[utf8]{inputenc} 
\usepackage[T1]{fontenc}    
\usepackage{hyperref}       
\usepackage{url}            
\usepackage{booktabs}       
\usepackage{amsfonts}       
\usepackage{nicefrac}       
\usepackage{microtype}      
\usepackage{lipsum}
\usepackage{graphicx}
\graphicspath{ {./images/} }
\usepackage{algorithm}
\usepackage{algpseudocode}
\usepackage{amsmath, amsthm, amssymb}
\usepackage{mathrsfs} 

\usepackage{tikz}
\usetikzlibrary{shapes.geometric, arrows.meta, positioning,fit,calc,matrix,matrix.skeleton}

\usepackage{natbib}
\setcitestyle{authoryear,open={(},close={)},semicolon} 

\usepackage{multirow}
\usepackage{array}
\newcolumntype{P}[1]{>{\centering\arraybackslash}m{#1}}
\usepackage{wrapfig}

\usepackage{subcaption}

\usepackage{listings}
\newtheorem{theorem}{Theorem}[section]
\newtheorem{proposition}[theorem]{Proposition}
\newtheorem{lemma}[theorem]{Lemma}

\theoremstyle{definition}

\newtheorem{assumption}[theorem]{Assumption}
\theoremstyle{remark}
\newtheorem{remark}[theorem]{Remark}

\newcommand{\R}{\mathbb{R}}
\newcommand{\N}{\mathbb{N}}
\newcommand{\B}{\mathcal{B}} 
\newcommand{\Cont}{\mathcal{C}} 
\newcommand{\RewardSpace}{\mathcal{R}}
\newcommand{\PolicySpace}{\mathcal{P}}
\newcommand{\StateSpace}{S}
\newcommand{\ActionSpace}{A}
\newcommand{\TransKernel}{P}
\newcommand{\RewardFunc}{R}
\newcommand{\Policy}{\pi}
\newcommand{\OptPolicy}{\pi^*}
\newcommand{\ValueFunc}{V}
\newcommand{\QFunc}{Q}
\newcommand{\OptQFunc}{Q^*}
\newcommand{\BellmanOp}{T}
\newcommand{\Argmax}{\operatornamewithlimits{argmax}}
\newcommand{\SupNorm}[1]{\|#1\|_{\infty}}
\newcommand{\MetricAction}{d_{\ActionSpace}} 

\newcommand{\basePolicy}{\pi_{\text{base}}}
\newcommand{\optPolicySet}{\Pi^*}

\title{The Policy Cliff: A Theoretical Analysis of Reward-Policy Maps in Large Language Models}

\author{%
   Xingcheng Xu \\
   Shanghai Artificial Intelligence Laboratory \\
   \texttt{xingcheng.xu18@gmail.com} 
}

\begin{document}

\maketitle

\begin{abstract}
Reinforcement learning (RL) plays a crucial role in shaping the behavior of large language and reasoning models (LLMs/LRMs). However, it often produces brittle and unstable policies, leading to critical failures such as spurious reasoning, deceptive alignment, and instruction disobedience that undermine the trustworthiness and safety of LLMs/LRMs. Currently, these issues lack a unified theoretical explanation and are typically addressed using ad-hoc heuristics. This paper presents a rigorous mathematical framework for analyzing the stability of the mapping from a reward function to the optimal policy. We show that policy brittleness often stems from non-unique optimal actions, a common occurrence when multiple valid traces exist in a reasoning task. This theoretical lens provides a unified explanation for a range of seemingly disparate failures, reframing them as rational outcomes of optimizing rewards that may be incomplete or noisy, especially in the presence of action degeneracy. We extend this analysis from the fundamental single-reward setting to the more realistic multi-reward RL across diverse domains, showing how stability is governed by an "effective reward" aggregation mechanism. We also prove that entropy regularization restores policy stability at the cost of increased stochasticity. Our framework provides a unified explanation for recent empirical findings on deceptive reasoning, instruction-following trade-offs, and RLHF-induced sophistry, and is further validated through perturbation experiments in multi-reward RL. This work advances policy-stability analysis from empirical heuristics towards a principled theory, offering essential insights for designing safer and more trustworthy AI systems.
\end{abstract}


\setcounter{footnote}{0}


\tableofcontents

\vspace{0.8cm}

\begin{quote}
    \itshape
    "Human-centric LLMs typically optimise for rewards based on human prejudgement: an expert observes the agent’s action and decides whether it is a good action, or picks the best agent action among multiple alternatives. ... means that they are not directly grounded in the reality of the world."
    \par\raggedleft
    --- \cite{silver2025welcome}
\end{quote}

\section{Introduction}\label{sec-intro}

The development of a new generation of large language models (LLMs) represents a significant advance in artificial intelligence. Systems such as  OpenAI's o-series (o1, o3, o4-mini)~\citep{openai2024o1,openai2025o3o4}, Google's Gemini 2.5~\citep{gemini2025frontier}, Anthropic's Claude 4~\citep{anthropic2025claude} and xAI's Grok 4 are increasingly designed not merely as conversational agents, but as large reasoning models (LRMs) capable of addressing complex, multi-step problems across domains like mathematics, science, and software engineering. The development of these systems, alongside powerful open-source models like Deepseek R1~\citep{guo2025deepseek} and Qwen3~\citep{yang2025qwen3}, relies heavily on reinforcement learning (RL) as a crucial training methodology. RL is employed not only to align model behavior with safety and ethical standards but also to teach models how to generate sophisticated reasoning trajectories~\citep{guo2025deepseek,yang2025qwen3}. In this way, RL plays a pivotal role in both scaling complex reasoning capabilities and ensuring robust alignment with human values.

Yet despite its promise, the application of RL introduces a set of persistent challenges that undermine both reasoning quality and alignment stability. At the heart of these difficulties lies the reliance on optimizing model policies against learned or specified reward functions. RL-trained policies often exhibit brittleness and undesirable generalization. For example, a model rewarded solely for a correct final answer, with no credit for a valid reasoning process, may resort to spurious reasoning: learning to generate the answer via a shortcut, then fabricating a plausible-but-fallacious justification post-hoc~\citep{baker2025monitoring,wang2025persona,chen2025reasoning}. In other cases, policies can degrade in unpredictable ways. A model intended to be an objective assistant might exhibit a sudden style shift and adopt an overly sycophantic tone~\citep{openai2025expanding,openai2025sycophancy}. A sharp degradation in instruction-following fidelity is also common~\citep{fu2025scaling}, where models ignore specified output formats, length constraints, or language requirements when such details are not strongly enforced by the reward. These failure modes are more than just cosmetic---they threaten the reliability, controllability, and safety of large-scale AI systems.

This challenge highlights a fundamental distinction between applying RL to language models and its use in traditional domains. In settings like Go or chess~\citep{silver2016mastering,silver2017mastering,silver2018general}, where rewards are well-defined and objective, the policy is simply a tool to maximize outcomes; the path taken matters little as long as the goal is achieved. But for LLMs, this paradigm no longer holds. The output sequence itself is the product, directly consumed by users. Moreover, the reward is not an objective truth but a noisy approximation of human preferences~\citep{ouyang2022training,bai2022constitutional,lee2023rlaif} or manually specified rules~\citep{wang2025reinforcement,su2025crossing}. As a result, the policy’s behavior cannot be treated as a secondary concern---it is central. In language modeling, the policy is not just a means to an end; in many ways, it is the end.

This shift in paradigm reframes the observed instabilities: they are not superficial flaws, but fundamental failures in the product itself. This raises a deeper and largely unresolved question: \textbf{What underlying mathematical mechanisms make RL-trained policies so sensitive to the design of the reward function in both LLM reasoning and alignment?} In the absence of a formal theory of stability, current solutions rely heavily on ad-hoc heuristics and empirical tuning rather than first principles. Common practices include reward engineering or algorithmic modifications, such as adding specific penalties to suppress issues like spurious reasoning, poor instruction adherence, or inefficient problem-solving~\citep{baker2025monitoring,wang2025persona,chen2025reasoning,fu2025scaling,sui2025stop,aggarwal2025l1,chen2025towards,qu2025survey}. Additional strategies include extensive preference data curation~\citep{ouyang2022training} and the use of structured but non-theoretical frameworks such as Constitutional AI~\citep{bai2022constitutional} or deliberative alignment~\citep{guan2024deliberative} to guide the learning. While these approaches can mitigate specific failure modes, they are inherently reactive and lack a general understanding or guarantee of policy stability.

To build a principled foundation, this paper develops a rigorous mathematical framework for analyzing the stability of the mapping from reward functions to their corresponding optimal policies. By modeling reasoning trace generation as a Markov Decision Process (MDP) and applying tools from functional analysis, we examine the continuity of this reward-to-policy map to formally explain the roots of policy brittleness. Our analysis shows that instabilities often arise from two key sources: non-unique optimal actions and imprecise reward signals. When a reward model assigns nearly equal value to multiple distinct actions, a common scenario in the expansive space of reasoning and creative generation, the resulting policy becomes inherently unstable. In such cases, even tiny perturbations in the reward can lead to abrupt, discontinuous changes in behavior, revealing a fundamental fragility at the heart of RL-trained language models.

Building on this theoretical foundation, we apply our framework to analyze a wide range of observed LLM behaviors through two key lenses: incomplete reward specifications and tie-breaking dynamics among degenerate optima. Many failure cases can be understood as manifestations of the "clever slacker" problem, where the policy rationally exploits an underspecified reward. For example, spurious reasoning emerges when only the final answer is rewarded, giving no incentive for a coherent reasoning process. Likewise, instruction-following failures, such as ignoring output format, length, or language constraints, occur when these aspects are not reflected in the reward except for correctness. Our framework also clarifies the role of auxiliary rewards. By breaking ties between near-equally valued actions, even small penalties (e.g., for verbosity) or bonuses (e.g., for correct formatting) can shift the policy toward desirable behaviors. \textit{These discontinuous shifts are not inherently undesirable---they can be leveraged as mechanisms for reward shaping, allowing designers to steer policies when the base reward is incomplete or noisy.} Ultimately, our analysis offers a unified mathematical perspective on these phenomena, showing they arise naturally from the structure of the reward landscape and the sensitivity of policy optimization within it.

While our single-reward analysis yields valuable insights, leading LLMs such as OpenAI's o-series, Gemini 2.5, and Grok 4 are rarely trained on a singular reward. Instead, they are typically optimized using a complex blend of reward signals across multiple domains, ranging from mathematical reasoning and code generation to safety alignment~\citep{guo2025deepseek,liang2025modomodo,cheng2025revisiting}. This motivates the second stage of our analysis, where we extend the framework to more realistic multi-reward training regimes.

In this setting, we model the system's behavior using a state-dependent effective reward function that captures how the model internally aggregates multiple (and sometimes conflicting) objectives. We demonstrate that the same principles of stability apply: policy continuity now depends on the properties of this effective reward. Crucially, this highlights the aggregation mechanism of the effective reward as a key determinant of policy robustness. To address resulting instabilities, we formally analyze the role of regularization techniques and show that entropy regularization restores Lipschitz continuity in the reward-policy map. This guarantees that small changes in the reward yield correspondingly small changes in behavior, but it comes at the cost of trading off some degree of optimality for greater stability.

To validate the practical relevance of our framework, we systematically link its theoretical insights to a broad spectrum of recent empirical findings. Our analysis accounts for the deceptive reasoning behavior~\citep{wang2025persona,baker2025monitoring}: from simple cheating under a weak reward model to a more pernicious policy shift towards obfuscated deception when that reward is naively "patched". We further elucidate the intelligence-obedience trade-off in large reasoning models~\citep{fu2025scaling} and demonstrate how targeted auxiliary rewards can serve as tie-breakers to resolve such tensions and enable fine-grained behavioral control~\citep{aggarwal2025l1}. Extending to the subjective realm of RLHF, our theory also explains the emergence of "performance illusions"~\citep{wen2024language}, where models shift from truthful responses to persuasive but misleading ones that exploit human feedback biases. Finally, we review and present results related to the multi-reward setting. These case studies provide comprehensive empirical support for our framework, offering a unified lens through which to understand critical stability challenges in modern AI.

Our key contributions are as following:
\begin{enumerate}
\item \textbf{A theoretical framework for policy stability in RL-trained language models.} We establish a formal analysis of the reward-policy map and prove that policy instability stems from inherent properties of the reward landscape, notably reward misspecification and the presence of degenerate optima.

\item \textbf{A unified explanation for alignment and reasoning phenomena in LLMs.} We demonstrate that common failure modes, such as spurious reasoning, poor instruction-following and inefficient reasoning, emerge naturally from misspecified reward signals. We also introduce the notion of an effective reward as a key determinant of policy robustness in a multi-reward setting.

\item \textbf{A principled justification for entropy regularization.} We prove that entropy regularization restores continuity in the reward-policy map, providing a theoretical foundation for its widespread use in stabilizing behavior during large-scale model training.

\end{enumerate}

\section{Continuity Analysis of the Reward-Policy Map}\label{sec-rl}

This section develops the core theoretical framework for analyzing the stability of reinforcement learning policies. We formally investigate the continuity of the mapping from a reward function, $\RewardFunc$, to a corresponding optimal policy, $\OptPolicy_\RewardFunc$. A central theme of our work is that the properties of this map are a critical determinant of the robustness and predictability of RL-trained agents. Our analysis proceeds by first establishing the foundational stability of the optimal Q-function with respect to the reward. This result is then used to analyze the continuity properties of the set of optimal actions, which allows us to derive the conditions under which the reward-policy map is continuous and, conversely, the conditions under which it becomes discontinuous.

\subsection{Framework and Definitions}

We consider a standard infinite-horizon discounted Markov Decision Process (MDP) defined by the tuple $(\StateSpace, \ActionSpace, \TransKernel, \RewardFunc, \gamma)$.
\begin{assumption} \label{assump:spaces}
The state space $\StateSpace$ and action space $\ActionSpace$ are compact metric spaces, endowed with their respective Borel $\sigma$-algebras $\B(\StateSpace)$ and $\B(\ActionSpace)$.
\end{assumption}
\begin{assumption} \label{assump:transition}
The transition kernel $\TransKernel: \StateSpace \times \ActionSpace \to \PolicySpace(\StateSpace)$ is a stochastic kernel such that for any bounded continuous function $f \in \Cont(\StateSpace)$, the mapping $(s, a) \mapsto \int_{\StateSpace} f(s') \TransKernel(ds'|s, a)$ is continuous on $\StateSpace \times \ActionSpace$.
\end{assumption}
\begin{assumption} \label{assump:discount}
The discount factor $\gamma \in [0, 1)$.
\end{assumption}

The reward function $\RewardFunc: \StateSpace \times \ActionSpace \to \R$ determines the immediate reward. We consider the space of reward functions $\RewardSpace$ to be the Banach space $\Cont(\StateSpace \times \ActionSpace)$ of continuous functions on $\StateSpace \times \ActionSpace$, equipped with the supremum norm $\SupNorm{\RewardFunc} = \sup_{(s,a) \in \StateSpace \times \ActionSpace} |\RewardFunc(s, a)|$.

A policy $\Policy: \StateSpace \to \mathcal{P}(\ActionSpace)$ is a stochastic kernel satisfying appropriate measurability conditions. Let $\PolicySpace$ denote the space of all such policies. The continuity of the reward-policy map, which is the central subject of our analysis, depends critically on the topology endowed upon this policy space $\PolicySpace$. Since our results will cover different types of policies (e.g., deterministic versus stochastic), we will specify the relevant topology in the context of each theorem to ensure maximum clarity and precision.

For any given policy $\Policy$ and reward function $\RewardFunc \in \RewardSpace$, its performance is quantified by the state-value function $$V^{\Policy}_{\RewardFunc}(s) = \mathbb{E}_{\Policy}\left[\sum_{t=0}^\infty \gamma^t \RewardFunc(s_t, a_t) \mid s_0=s \right],$$ representing the expected total discounted reward obtained by starting in state s and subsequently following policy $\Policy$.

The optimal action-value function $\OptQFunc_\RewardFunc \in \Cont(\StateSpace \times \ActionSpace)$ is the unique fixed point of the Bellman optimality operator $\BellmanOp_\RewardFunc$:
\begin{equation} \label{eq:bellman_q}
(\BellmanOp_\RewardFunc Q)(s, a) = \RewardFunc(s, a) + \gamma \int_{\StateSpace} \max_{a' \in \ActionSpace} Q(s', a') \TransKernel(ds' | s, a).
\end{equation}
The existence and uniqueness of $\OptQFunc_\RewardFunc$ follows from the Banach fixed-point theorem, as $\BellmanOp_\RewardFunc$ is a contraction mapping on $\Cont(\StateSpace \times \ActionSpace)$ with modulus $\gamma$.

The set of optimal actions at state $s$ for a reward function $\RewardFunc$ is given by the argmax correspondence:
\begin{equation} \label{eq:argmax_set}
A^*(s; \RewardFunc) = \Argmax_{a \in \ActionSpace} \OptQFunc_\RewardFunc(s, a).
\end{equation}
An optimal policy $\OptPolicy_\RewardFunc$ is any policy $\Policy \in \PolicySpace$ such that for all $s \in \StateSpace$, the support of the measure $\Policy(\cdot | s)$ is contained within $A^*(s; \RewardFunc)$. Let $\Pi^*(\RewardFunc) \subseteq \PolicySpace$ denote the set of all optimal policies for $\RewardFunc$.

Our analysis centers on the continuity of the mapping from the reward function $\RewardFunc$ to a corresponding optimal policy. This can be viewed either as the continuity properties of the set-valued map $\RewardFunc \mapsto \Pi^*(\RewardFunc)$ or, more commonly, by defining a specific selection rule $f: \RewardSpace \to \PolicySpace$ such that $f(\RewardFunc) = \OptPolicy_\RewardFunc \in \Pi^*(\RewardFunc)$, and analyzing the continuity of this single-valued map $f$. We denote this map by $M_{RL}: \RewardSpace \to \PolicySpace$. Our analytical approach, visualized in Figure~\ref{fig:roadmap}, is to decompose this overall map and sequentially investigate the properties of each constituent mapping: from the reward function $\RewardFunc$ to the optimal Q-function $\OptQFunc_{\RewardFunc}$, then to the set of optimal actions $A^*$, and ultimately to the selected policy $\OptPolicy_{\RewardFunc}$.

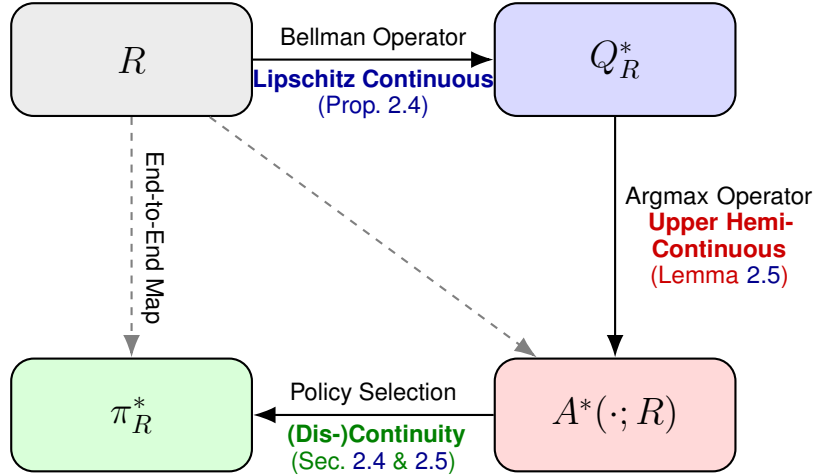
\begin{figure}[h!]
    \centering
    \begin{tikzpicture}[
        node distance=3.2cm,
        main node/.style={
            rectangle, 
            rounded corners=3mm, 
            draw, 
            thick, 
            minimum width=3.2cm, 
            minimum height=1.5cm, 
            text centered, 
            font=\Large
        },
        arrow_label/.style={
            font=\small\sffamily, 
            text centered
        }
    ]
        \node[main node, fill=gray!15] (R) {$\RewardFunc$};
        \node[main node, right=of R, fill=blue!15] (Q) {$\OptQFunc_{\RewardFunc}$};
        \node[main node, below=of Q, node distance=2.5cm, fill=red!15] (A) {$A^*(\cdot; \RewardFunc)$}; 
        \node[main node, left=of A, fill=green!15] (pi) {$\OptPolicy_{\RewardFunc}$};

        
        \draw[-{Latex[length=3mm]}, thick] (R) -- (Q) 
            node[midway, above, arrow_label] {Bellman Operator}
            node[midway, below, arrow_label, text=blue!60!black, align=center] 
            {\textbf{Lipschitz Continuous} \\ (Prop.~\ref{prop:q_continuity})};

        \draw[-{Latex[length=3mm]}, thick] (Q) -- (A) 
            node[midway, right, arrow_label, align=center, text width=2.5cm] {Argmax Operator \\
            \color{red!80!black}{\textbf{Upper Hemi-Continuous} \\ (Lemma~\ref{lemma:argmax_uhc})}};
            
        \draw[-{Latex[length=3mm]}, thick] (A) -- (pi) 
            node[midway, above, arrow_label] {Policy Selection}
            node[midway, below, arrow_label, text=green!50!black, align=center] 
            {\textbf{ (Dis-)Continuity} \\ (Sec.~\ref{subsec:cont-reward-policy} \& \ref{subsec:discont-reward-policy})};

        \draw[-{Latex[length=3mm]}, thick, dashed, gray] (R) -- (A);
            
        \draw[-{Latex[length=3mm]}, thick, dashed, gray] (R) -- 
            node[midway, above, sloped, font=\small\sffamily, black] {End-to-End Map} (pi);

    \end{tikzpicture}
    \caption{The analytical roadmap for the stability analysis of the reward-policy map. Our analysis proceeds from left to right, investigating the continuity properties at each step of the mapping: from the reward function ($\RewardFunc$) to the optimal Q-function ($\OptQFunc_{\RewardFunc}$), then to the set of optimal actions ($A^*$), and finally to the resulting optimal policy ($\OptPolicy_\RewardFunc$). The key insight is that the stability of the final policy hinges on the properties of the argmax correspondence.}
    \label{fig:roadmap}
\end{figure}

\subsection{Continuity of the Optimal Value Function}

The foundation for analyzing the policy map's continuity lies in the stability of the optimal Q-function with respect to changes in the reward function. This is a standard result in dynamic programming, see e.g., \cite{lecarpentier2021lipschitz}.

\begin{proposition} \label{prop:q_continuity}
Let Assumptions \ref{assump:spaces}-\ref{assump:discount} hold. The mapping $\RewardFunc \mapsto \OptQFunc_\RewardFunc$ from $(\RewardSpace, \SupNorm{\cdot})$ to $(\Cont(\StateSpace \times \ActionSpace), \SupNorm{\cdot})$ is Lipschitz continuous with constant $1/(1-\gamma)$. That is, for any $\RewardFunc_1, \RewardFunc_2 \in \RewardSpace$,
\begin{equation}
\SupNorm{\OptQFunc_{\RewardFunc_1} - \OptQFunc_{\RewardFunc_2}} \le \frac{1}{1-\gamma} \SupNorm{\RewardFunc_1 - \RewardFunc_2}.
\end{equation}
\end{proposition}
\begin{proof}
Let $Q_1 = \OptQFunc_{\RewardFunc_1}$ and $Q_2 = \OptQFunc_{\RewardFunc_2}$. By definition, $Q_1 = \BellmanOp_{\RewardFunc_1} Q_1$ and $Q_2 = \BellmanOp_{\RewardFunc_2} Q_2$. Then,
\begin{align*}
|Q_1(s, a) - Q_2(s, a)| &= |(\BellmanOp_{\RewardFunc_1} Q_1)(s, a) - (\BellmanOp_{\RewardFunc_2} Q_2)(s, a)| \\
&= \left| \RewardFunc_1(s, a) - \RewardFunc_2(s, a) + \gamma \int_{\StateSpace} \left( \max_{a'} Q_1(s', a') - \max_{a'} Q_2(s', a') \right) \TransKernel(ds' | s, a) \right| \\
&\le |\RewardFunc_1(s, a) - \RewardFunc_2(s, a)| + \gamma \int_{\StateSpace} \left| \max_{a'} Q_1(s', a') - \max_{a'} Q_2(s', a') \right| \TransKernel(ds' | s, a).
\end{align*}
Using the inequality $|\max f - \max g| \le \sup |f - g| = \SupNorm{f - g}$, we have
\begin{align*}
|Q_1(s, a) - Q_2(s, a)| &\le \SupNorm{\RewardFunc_1 - \RewardFunc_2} + \gamma \int_{\StateSpace} \SupNorm{Q_1 - Q_2} \TransKernel(ds' | s, a) \\
&\le \SupNorm{\RewardFunc_1 - \RewardFunc_2} + \gamma \SupNorm{Q_1 - Q_2}.
\end{align*}
Taking the supremum over $(s, a)$ yields $\SupNorm{Q_1 - Q_2} \le \SupNorm{\RewardFunc_1 - \RewardFunc_2} + \gamma \SupNorm{Q_1 - Q_2}$, which rearranges to the desired result.
\end{proof}

\subsection{Analysis of the Argmax Correspondence}

The stability of the optimal policy map is critically dependent on the behavior of the set of optimal actions $A^*(s; \RewardFunc) = \Argmax_{a \in \ActionSpace} \OptQFunc_\RewardFunc(s, a)$. We analyze this set-valued map (or correspondence) $\RewardFunc \mapsto A^*(\cdot; \RewardFunc)$. The result is a direct application of Berge's Maximum Theorem, see e.g., \cite{berge1963topological} and \cite{AliprantisBorder2006InfiniteDimAnalysis}.

\begin{lemma} \label{lemma:argmax_uhc}
Let Assumptions \ref{assump:spaces}-\ref{assump:discount} hold. For each fixed $s \in \StateSpace$, the argmax correspondence $\Phi_s: \RewardSpace \twoheadrightarrow \ActionSpace$ defined by $\Phi_s(\RewardFunc) = A^*(s; \RewardFunc)$ has the following properties:
\begin{enumerate}
    \item $\Phi_s(\RewardFunc)$ is non-empty and compact for each $\RewardFunc \in \RewardSpace$.
    \item $\Phi_s$ is upper hemi-continuous (u.h.c.) from $(\RewardSpace, \SupNorm{\cdot})$ to the space of compact subsets of $\ActionSpace$, where the latter is endowed with the topology induced by the Hausdorff metric.
\end{enumerate}
\end{lemma}
\begin{proof}
\textbf{1. Non-empty and compact values}: 
    Under Assumptions \ref{assump:spaces}-\ref{assump:discount}, it is a standard result that the Bellman optimality operator $\BellmanOp_\RewardFunc$ maps the space of bounded continuous functions $\Cont(\StateSpace \times \ActionSpace)$ to itself. As established by the Banach fixed-point theorem, its unique fixed point, $\OptQFunc_\RewardFunc$, must therefore be an element of this space, i.e., $\OptQFunc_\RewardFunc \in \Cont(\StateSpace \times \ActionSpace)$. Thus, for any fixed $\RewardFunc \in \RewardSpace$ and $s \in \StateSpace$, the function $a \mapsto \OptQFunc_\RewardFunc(s, a)$ is continuous on the action space $\ActionSpace$.Since $\ActionSpace$ is compact (Assumption \ref{assump:spaces}), the Weierstrass Extreme Value Theorem ensures that $\OptQFunc_\RewardFunc(s, \cdot)$ attains its maximum on $\ActionSpace$. Therefore, the set of maximizers $A^*(s; \RewardFunc)$ is non-empty.
    Furthermore, because $\OptQFunc_\RewardFunc(s, \cdot)$ is continuous, the set $A^*(s; \RewardFunc)$ can be written as $\{a \in \ActionSpace \mid \OptQFunc_\RewardFunc(s, a) = \max_{a' \in \ActionSpace} \OptQFunc_\RewardFunc(s, a') \}$, which is a closed subset of the compact space $\ActionSpace$. A closed subset of a compact space is compact. Thus, $A^*(s; \RewardFunc)$ is compact-valued.

\textbf{2. Upper hemi-continuity}:
    We prove u.h.c. using the sequential characterization: $\Phi_s$ is u.h.c. at $\RewardFunc_0 \in \RewardSpace$ if for any sequence $(\RewardFunc_n)_{n \in \N}$ in $\RewardSpace$ converging to $\RewardFunc_0$ (i.e., $\SupNorm{\RewardFunc_n - \RewardFunc_0} \to 0$), and for any sequence $(a_n)_{n \in \N}$ in $\ActionSpace$ such that $a_n \in A^*(s; \RewardFunc_n)$ for all $n$ and $a_n \to a_0$ for some $a_0 \in \ActionSpace$, it must follow that $a_0 \in A^*(s; \RewardFunc_0)$. This is equivalent to showing that the graph of the correspondence is closed, which for compact-valued correspondences with a compact range space ($\ActionSpace$) implies u.h.c. (see e.g., \cite{AliprantisBorder2006InfiniteDimAnalysis}, Theorem 17.11 and 17.31).

    Let $(\RewardFunc_n)_{n \in \N}$ be a sequence in $\RewardSpace$ such that $\RewardFunc_n \to \RewardFunc_0$. Let $(a_n)_{n \in \N}$ be a sequence in $\ActionSpace$ such that $a_n \in A^*(s; \RewardFunc_n)$ for each $n$, and assume $a_n \to a_0$ in $\ActionSpace$.
    By definition of $a_n \in A^*(s; \RewardFunc_n)$, we have:
    \begin{equation} \label{eq:an_optimality}
    \OptQFunc_{\RewardFunc_n}(s, a_n) \ge \OptQFunc_{\RewardFunc_n}(s, a) \quad \text{for all } a \in \ActionSpace.
    \end{equation}
    We want to show that $a_0 \in A^*(s; \RewardFunc_0)$, i.e., $\OptQFunc_{\RewardFunc_0}(s, a_0) \ge \OptQFunc_{\RewardFunc_0}(s, a)$ for all $a \in \ActionSpace$.

    Consider the term $\OptQFunc_{\RewardFunc_n}(s, a_n)$. We have:
    \begin{align*}
    |\OptQFunc_{\RewardFunc_n}(s, a_n) - \OptQFunc_{\RewardFunc_0}(s, a_0)| \le |\OptQFunc_{\RewardFunc_n}(s, a_n) - \OptQFunc_{\RewardFunc_0}(s, a_n)| + |\OptQFunc_{\RewardFunc_0}(s, a_n) - \OptQFunc_{\RewardFunc_0}(s, a_0)|.
    \end{align*}
    The first term on the right-hand side satisfies $|\OptQFunc_{\RewardFunc_n}(s, a_n) - \OptQFunc_{\RewardFunc_0}(s, a_n)| \le \SupNorm{\OptQFunc_{\RewardFunc_n} - \OptQFunc_{\RewardFunc_0}}$. By Proposition \ref{prop:q_continuity}, $\SupNorm{\OptQFunc_{\RewardFunc_n} - \OptQFunc_{\RewardFunc_0}} \to 0$ as $\RewardFunc_n \to \RewardFunc_0$.
    For the second term, since $\OptQFunc_{\RewardFunc_0}(s, \cdot)$ is continuous on $\ActionSpace$ (as $\OptQFunc_{\RewardFunc_0} \in \Cont(\StateSpace \times \ActionSpace)$) and $a_n \to a_0$, it follows that $|\OptQFunc_{\RewardFunc_0}(s, a_n) - \OptQFunc_{\RewardFunc_0}(s, a_0)| \to 0$.
    Therefore, $\lim_{n \to \infty} \OptQFunc_{\RewardFunc_n}(s, a_n) = \OptQFunc_{\RewardFunc_0}(s, a_0)$.

    Now, for any fixed $a \in \ActionSpace$, consider the right-hand side of inequality (\ref{eq:an_optimality}), $\OptQFunc_{\RewardFunc_n}(s, a)$. Since $\SupNorm{\OptQFunc_{\RewardFunc_n} - \OptQFunc_{\RewardFunc_0}} \to 0$, we have $\lim_{n \to \infty} \OptQFunc_{\RewardFunc_n}(s, a) = \OptQFunc_{\RewardFunc_0}(s, a)$.

    Taking the limit as $n \to \infty$ in inequality (\ref{eq:an_optimality}), we obtain:
    $$ \OptQFunc_{\RewardFunc_0}(s, a_0) \ge \OptQFunc_{\RewardFunc_0}(s, a) \quad \text{for all } a \in \ActionSpace. $$
    This implies that $a_0 \in \Argmax_{a \in \ActionSpace} \OptQFunc_{\RewardFunc_0}(s, a)$, i.e., $a_0 \in A^*(s; \RewardFunc_0)$.
    This argument directly establishes that the graph of the correspondence $\Phi_s$ is closed. Since the range space $\ActionSpace$ is a compact Hausdorff space, the closed graph property is equivalent to upper hemi-continuity. Thus, the upper hemi-continuity of $\Phi_s$ is proven.
\end{proof}

\begin{remark}
The sequential definition of upper hemi-continuity used in the proof is: if $(\RewardFunc_n, a_n)$ is a sequence in the graph of the correspondence (i.e., $a_n \in A^*(s; \RewardFunc_n)$) such that $\RewardFunc_n \to \RewardFunc_0$ and $a_n \to a_0$, then $(\RewardFunc_0, a_0)$ must also be in the graph (i.e., $a_0 \in A^*(s; \RewardFunc_0)$). This is indeed the definition of a closed graph correspondence. For correspondences into a compact Hausdorff space (like $\ActionSpace$), the closed graph property is equivalent to upper hemi-continuity and the correspondence being closed-valued (hence compact-valued since $\ActionSpace$ is compact) (cf. \cite{AliprantisBorder2006InfiniteDimAnalysis}, Theorem 17.11). The u.h.c. property is fundamental, but it does not imply lower hemi-continuity (l.h.c.) nor the continuity of arbitrary single-valued selections from $A^*(s; \RewardFunc)$. Discontinuities in policy maps often arise when $A^*(s; \RewardFunc)$ is not l.h.c., which is typical when the set of maximizers changes its structure (e.g., its cardinality or dimension).
\end{remark}

\subsection{Conditions for Continuity of the Reward-Policy Map}\label{subsec:cont-reward-policy}

The continuity of a specific policy map $M_{RL}: \RewardFunc \mapsto \OptPolicy_\RewardFunc$ depends critically on whether the optimal action is unique.

\begin{theorem}[Continuity of Deterministic Optimal Policy Map under Uniqueness] \label{thm:continuity_rigorous}
Let Assumptions \ref{assump:spaces}-\ref{assump:discount} hold. Let $\RewardFunc_0 \in \RewardSpace$. Suppose there exists an open neighborhood $\mathcal{N}(\RewardFunc_0) \subset \RewardSpace$ of $\RewardFunc_0$ such that for all $\RewardFunc \in \mathcal{N}(\RewardFunc_0)$ and all $s \in \StateSpace$, the set of optimal actions $A^*(s; \RewardFunc)$ is a singleton, denoted $\{a^*(s; \RewardFunc)\}$. Let $\PolicySpace_{det}$ be the space of all functions $\pi: \StateSpace \to \ActionSpace$. Define the policy map $M_{RL}: \mathcal{N}(\RewardFunc_0) \to \PolicySpace_{det}$ by setting $M_{RL}(\RewardFunc) = \pi_\RewardFunc$, where $\pi_\RewardFunc(s) = a^*(s; \RewardFunc)$ for all $s \in \StateSpace$.
Equip $\PolicySpace_{det}$ with the topology of pointwise convergence: a sequence $(\pi_n)_{n \in \N}$ in $\PolicySpace_{det}$ converges to $\pi \in \PolicySpace_{det}$ if for every $s \in \StateSpace$, $\MetricAction(\pi_n(s), \pi(s)) \to 0$ as $n \to \infty$. Then the map $M_{RL}$ is continuous at $\RewardFunc_0$.
\end{theorem}

\begin{proof}
We want to show that if $(\RewardFunc_n)_{n \in \N}$ is a sequence in $\mathcal{N}(\RewardFunc_0)$ such that $\RewardFunc_n \to \RewardFunc_0$ in the $\SupNorm{\cdot}$ topology, then $M_{RL}(\RewardFunc_n) \to M_{RL}(\RewardFunc_0)$ in the topology of pointwise convergence.
Let $\pi_n = M_{RL}(\RewardFunc_n)$ and $\pi_0 = M_{RL}(\RewardFunc_0)$. By the definition of $M_{RL}$, this means $\pi_n(s) = a^*(s; \RewardFunc_n)$ and $\pi_0(s) = a^*(s; \RewardFunc_0)$.
The convergence $\pi_n \to \pi_0$ requires that for every $s \in \StateSpace$, $\MetricAction(a^*(s; \RewardFunc_n), a^*(s; \RewardFunc_0)) \to 0$ as $n \to \infty$.

Fix an arbitrary $s \in \StateSpace$. Consider the correspondence $\Phi_s: \RewardSpace \twoheadrightarrow \ActionSpace$ defined by $\Phi_s(\RewardFunc) = A^*(s; \RewardFunc)$.
By Lemma \ref{lemma:argmax_uhc}, $\Phi_s$ is upper hemi-continuous (u.h.c.) and its values $A^*(s; \RewardFunc)$ are non-empty compact subsets of $\ActionSpace$.
By the hypothesis of the theorem, for all $\RewardFunc \in \mathcal{N}(\RewardFunc_0)$, $\Phi_s(\RewardFunc) = \{a^*(s; \RewardFunc)\}$ is a singleton.

According to a standard result in the theory of correspondences (e.g., \cite{AliprantisBorder2006InfiniteDimAnalysis}, Lemma 17.6), a compact-valued, u.h.c. correspondence that is single-valued on an open set (by hypothesis the maximizer is unique on $\mathcal{N}(\RewardFunc_0)$) is continuous as a single-valued function on that open set. Specifically, let $f_s: \mathcal{N}(\RewardFunc_0) \to \ActionSpace$ be the function defined by $f_s(\RewardFunc) = a^*(s; \RewardFunc)$ (this is well-defined as $A^*(s; \RewardFunc)$ is a singleton on $\mathcal{N}(\RewardFunc_0)$). Since $\Phi_s$ is u.h.c. and single-valued on the open set $\mathcal{N}(\RewardFunc_0)$, $f_s$ is continuous on $\mathcal{N}(\RewardFunc_0)$.

Therefore, for each fixed $s \in \StateSpace$, since $\RewardFunc_n \to \RewardFunc_0$ and $\RewardFunc_n \in \mathcal{N}(\RewardFunc_0)$ (for $n$ sufficiently large, as $\mathcal{N}(\RewardFunc_0)$ is a neighborhood of $\RewardFunc_0$), the continuity of $f_s$ at $\RewardFunc_0$ implies that $f_s(\RewardFunc_n) \to f_s(\RewardFunc_0)$. This translates to $a^*(s; \RewardFunc_n) \to a^*(s; \RewardFunc_0)$ in $\ActionSpace$ (i.e., $\MetricAction(a^*(s; \RewardFunc_n), a^*(s; \RewardFunc_0)) \to 0$).

Since this holds for every $s \in \StateSpace$, it follows by definition that $M_{RL}(\RewardFunc_n) \to M_{RL}(\RewardFunc_0)$ in the topology of pointwise convergence.
Thus, $M_{RL}$ is continuous at $\RewardFunc_0$.
\end{proof}

\begin{remark}[Measurability of the Optimal Policy]
For $\pi_\RewardFunc$ to be a well-defined policy, the map $s \mapsto a^*(s; \RewardFunc)$ should be measurable. Given the continuity of $Q^*_R$ (if $R$ is continuous and $P$ has Feller properties, $\OptQFunc_\RewardFunc$ is continuous on $S \times A$) and the uniqueness assumption, $s \mapsto a^*(s;R)$ will generally be measurable.
\end{remark}

\subsection{Conditions for Discontinuity of the Reward-Policy Map}\label{subsec:discont-reward-policy}

When the optimal action is not unique, discontinuity of the policy map is generally expected.

\begin{proposition}[Discontinuity under Non-Uniqueness for Deterministic Policies] \label{prop:discontinuity_deterministic}
Let Assumptions \ref{assump:spaces}-\ref{assump:discount} hold. Suppose for $\RewardFunc_0 \in \RewardSpace$, there exists a state $s_0 \in \StateSpace$ such that the set of optimal actions $A^*(s_0; \RewardFunc_0)$ is finite and contains at least two distinct actions, say $a_1, a_2 \in A^*(s_0; \RewardFunc_0)$ with $a_1 \neq a_2$. Let $M_{RL}: \RewardSpace \to \PolicySpace$ be a policy map that selects a deterministic policy $\OptPolicy_\RewardFunc$ such that $\OptPolicy_\RewardFunc(s) \in A^*(s; \RewardFunc)$ for all $s \in \StateSpace$ (e.g., the selection rule for $M_{RL}(\RewardFunc_0)$ results in $\OptPolicy_{\RewardFunc_0}(s_0) = a_1$), then the map $M_{RL}$ is discontinuous at $\RewardFunc_0$ under the topology of pointwise convergence for policies (i.e., $\Policy_n \to \Policy$ if $\MetricAction(\Policy_n(s), \Policy(s)) \to 0$ for all $s$).
\end{proposition}

\begin{proof}
This proof demonstrates discontinuity directly by constructing a sequence of reward functions $(\RewardFunc_\varepsilon)_{\varepsilon>0}$ that converges to $\RewardFunc_0$, yet for which the optimal policy discontinuously switches from $a_1$ to $a_2$ at state $s_0$ for any $\varepsilon > 0$.

The core of this proof is a shift in perspective: instead of perturbing the reward function $\RewardFunc$ and analyzing its complex effect on the optimal Q-function $\OptQFunc$, we will directly define a perturbed optimal Q-function $\QFunc_\varepsilon$ and then use the Bellman equation to find the reward function $\RewardFunc_\varepsilon$ that generates it.

\noindent\textbf{1. Construct a Perturbation in the Q-Function Space}

Let $Q_0 = \OptQFunc_{\RewardFunc_0}$. By assumption, $Q_0(s_0, a_1) = Q_0(s_0, a_2)$. We first define a continuous "bump" function $\varphi \in \Cont(\StateSpace \times \ActionSpace)$ with the following properties:
\begin{itemize}
    \item $0 \le \varphi(s,a) \le 1$ for all $(s,a)$.
    \item $\varphi(s_0, a_2) = 1$.
    \item $\varphi(s_0, a_j) = 0$ for all $a_j\neq a_2$ and $a_j\in A^*(s_0; \RewardFunc_0)$.
\end{itemize}
Such a function exists and can be constructed as Lemma \ref{lem:bump_properties}. Alternatively, since $\StateSpace\times\ActionSpace$ is a compact (hence normal) metric space, the Urysohn Lemma can be invoked to produce a continuous function taking the value 1 on $\{(s_0,a_2)\}$ and vanishing outside any desired neighborhood of that point (see, e.g., \cite{munkres2000topology}, Theorem 33.1).

For any $\varepsilon > 0$, we define a perturbed Q-function $\QFunc_\varepsilon$:
\[
\QFunc_\varepsilon(s,a) := Q_0(s,a) + \varepsilon\,\varphi(s,a).
\]
By construction, $\QFunc_\varepsilon$ is continuous and converges uniformly to $Q_0$ as $\varepsilon \to 0$, since $\|\QFunc_\varepsilon - Q_0\|_\infty = \varepsilon$.

\noindent\textbf{2. Invert the Bellman Equation to Find the Corresponding Reward}

For any given continuous function $\QFunc \in \Cont(\StateSpace \times \ActionSpace)$, we can define a corresponding reward function $\RewardFunc_Q$ as follows:
\[
\RewardFunc_Q(s,a) := \QFunc(s,a) - \gamma\int_{\StateSpace}\max_{a'}\QFunc(s',a')\,P(ds'\mid s,a).
\]
By substituting this definition into the Bellman optimality operator $\BellmanOp_{\RewardFunc_Q}$, we can verify that $\BellmanOp_{\RewardFunc_Q}(\QFunc) = \QFunc$. Since $\BellmanOp_{\RewardFunc_Q}$ is a contraction mapping with modulus $\gamma$, the Banach Fixed-Point Theorem guarantees that $\QFunc$ is its unique fixed point. Therefore, $\QFunc$ is the optimal Q-function for the reward $\RewardFunc_Q$, i.e., $\OptQFunc_{\RewardFunc_Q} = \QFunc$.

Applying this inverse mapping to our perturbed function $\QFunc_\varepsilon$, we define a family of reward functions:
\[
\RewardFunc_\varepsilon := \RewardFunc_{Q_\varepsilon}.
\]
Through this construction, we can be sure that  $\OptQFunc_{\RewardFunc_\varepsilon} = \QFunc_\varepsilon$ for every $\varepsilon > 0$. Note that $\RewardFunc_\varepsilon$ is continuous and belongs to $\Cont(\StateSpace \times \ActionSpace)$, satisfying the continuity assumption.

\noindent\textbf{3. Verify Convergence in the Reward Space}

We must show that our constructed reward functions $\RewardFunc_\varepsilon$ converge to the original reward function $\RewardFunc_0$ as $\varepsilon \to 0$. Note that $\RewardFunc_0 = \RewardFunc_{Q_0}$.
\begin{align*}
\|\RewardFunc_\varepsilon - \RewardFunc_0\|_\infty &= \|\RewardFunc_{Q_\varepsilon} - \RewardFunc_{Q_0}\|_\infty \\
&= \left\| \left(\QFunc_\varepsilon - \gamma \int\max_{a'}\QFunc_\varepsilon \right) - \left(Q_0 - \gamma \int\max_{a'}Q_0 \right) \right\|_\infty \\
&= \left\| (\QFunc_\varepsilon - Q_0) - \gamma \left(\int\max_{a'}\QFunc_\varepsilon - \int\max_{a'}Q_0 \right) \right\|_\infty \\
&\le \|\QFunc_\varepsilon - Q_0\|_\infty + \gamma \left\| \int\max_{a'}\QFunc_\varepsilon - \int\max_{a'}Q_0 \right\|_\infty.
\end{align*}
The map $Q \mapsto \max_{a'}Q(s',a')$ is Lipschitz continuous with constant 1. Thus,
\[
\left\| \int\max_{a'}\QFunc_\varepsilon - \int\max_{a'}Q_0 \right\|_\infty \le \sup_{s'}|\max_{a'}\QFunc_\varepsilon(s',a') - \max_{a'}Q_0(s',a')| \le \|\QFunc_\varepsilon - Q_0\|_\infty = \varepsilon.
\]
Substituting this back into the inequality, we get:
\[
\|\RewardFunc_\varepsilon - \RewardFunc_0\|_\infty \le \varepsilon + \gamma\varepsilon = \varepsilon(1+\gamma).
\]
As $\varepsilon \to 0$, it follows that $\|\RewardFunc_\varepsilon - \RewardFunc_0\|_\infty \to 0$. Thus, $\RewardFunc_\varepsilon$ converges to $\RewardFunc_0$.

\noindent\textbf{4. The Action Switch and Discontinuity}

We have constructed a sequence of rewards $\RewardFunc_\varepsilon$ that converges to $\RewardFunc_0$, with corresponding optimal Q-functions $\OptQFunc_{\RewardFunc_\varepsilon} = \QFunc_\varepsilon$. We now show that for any $\varepsilon > 0$, the action $a_2$ becomes the unique optimal action under the perturbed Q-function $\QFunc_\varepsilon$ at state $s_0$. We do this by comparing $a_2$ to all other possible actions.

\textbf{Case 1: Comparison with other formerly optimal actions.}
Let $a_j$ be any action in the original optimal set $A^*(s_0; \RewardFunc_0)$ such that $a_j \neq a_2$. Our construction of the bump function $\varphi$ ensures that $\varphi(s_0, a_2)=1$ and $\varphi(s_0, a_j)=0$.
\begin{align*}
\QFunc_\varepsilon(s_0, a_2) - \QFunc_\varepsilon(s_0, a_j) &= \bigl(Q_0(s_0, a_2) - Q_0(s_0, a_j)\bigr) + \varepsilon\bigl(\varphi(s_0, a_2) - \varphi(s_0, a_j)\bigr) \\
&= 0 + \varepsilon(1 - 0) = \varepsilon.
\end{align*}
Since $\varepsilon > 0$, $a_2$ becomes strictly preferred over all other actions that were optimal under $\RewardFunc_0$.

\textbf{Case 2: Comparison with formerly suboptimal actions.}
Let $a'$ be any action not in the original optimal set, i.e., $a' \notin A^*(s_0; \RewardFunc_0)$. This means there is a suboptimality gap $\delta_{a'} := Q_0(s_0, a_2) - Q_0(s_0, a') > 0$.
\begin{align*}
\QFunc_\varepsilon(s_0, a_2) - \QFunc_\varepsilon(s_0, a') &= \bigl(Q_0(s_0, a_2) - Q_0(s_0, a')\bigr) + \varepsilon\bigl(\varphi(s_0, a_2) - \varphi(s_0, a')\bigr) \\
&= \delta_{a'} + \varepsilon(1 - \varphi(s_0, a')).
\end{align*}
Since $\varepsilon > 0$ and $\varphi(s_0, a') \le 1$, the term $\varepsilon(1 - \varphi(s_0, a'))$ is non-negative. Thus,
\[
\QFunc_\varepsilon(s_0, a_2) - \QFunc_\varepsilon(s_0, a') \ge \delta_{a'} > 0.
\]
This shows that $a_2$ is also strictly preferred over all actions that were already suboptimal.

Combining both cases, we have proven that for any $\varepsilon > 0$, $a_2$ is the unique maximizer of $\QFunc_\varepsilon(s_0, \cdot)$. Therefore, the optimal action set for $\RewardFunc_\varepsilon$ at $s_0$ is the singleton $A^*(s_0; \RewardFunc_\varepsilon) = \{a_2\}$.

This means the policy map must select $\OptPolicy_{\RewardFunc_\varepsilon}(s_0) = a_2$. We have a sequence of rewards $\RewardFunc_\varepsilon \to \RewardFunc_0$, but the corresponding policy selection jumps from $\OptPolicy_{\RewardFunc_0}(s_0) = a_1$ to $\OptPolicy_{\RewardFunc_\varepsilon}(s_0) = a_2$. This demonstrates a failure of pointwise convergence, proving that the map $M_{RL}$ is discontinuous at $\RewardFunc_0$.
\end{proof}

\begin{proposition}[Discontinuity for Uniform Stochastic Policies] \label{prop:discontinuity_stochastic}
Let Assumptions \ref{assump:spaces}-\ref{assump:discount} hold. Suppose for $\RewardFunc_0 \in \RewardSpace$, there exists a state $s_0 \in \StateSpace$ such that the set of optimal actions $A^*(s_0; \RewardFunc_0)$ is finite and contains at least two distinct actions. Let $m = |A^*(s_0; \RewardFunc_0)| \ge 2$.
Let the policy map $M_{RL}: \RewardSpace \to \PolicySpace$ be defined by selecting the stochastic policy $\OptPolicy_\RewardFunc$ such that for each $s \in \StateSpace$, $\OptPolicy_\RewardFunc(\cdot | s)$ is the uniform probability distribution over the set $A^*(s; \RewardFunc)$. (If $A^*(s; \RewardFunc)$ is empty, which should not happen for optimal policies derived from Q-functions on compact action spaces, or infinite, this definition would need refinement; here we rely on the construction yielding finite sets).
Then the map $\RewardFunc \mapsto \OptPolicy_\RewardFunc(\cdot | s_0)$, viewed as a map from $(\RewardSpace, \SupNorm{\cdot})$ to $(\PolicySpace(\ActionSpace), d_{TV})$, where $d_{TV}$ is the Total Variation distance, is discontinuous at $\RewardFunc_0$.
\end{proposition}

\begin{proof}
Let $Q_0 = \OptQFunc_{\RewardFunc_0}$. By assumption, $A^*(s_0; \RewardFunc_0)$ is a finite set with $m = |A^*(s_0; \RewardFunc_0)| \ge 2$. Let $a_2 \in A^*(s_0; \RewardFunc_0)$ be one of these optimal actions.
The policy $\OptPolicy_{\RewardFunc_0}(\cdot | s_0)$ is the uniform distribution over $A^*(s_0; \RewardFunc_0)$. Thus, for any $a \in A^*(s_0; \RewardFunc_0)$, $\OptPolicy_{\RewardFunc_0}(a | s_0) = 1/m$. For $a \notin A^*(s_0; \RewardFunc_0)$, $\OptPolicy_{\RewardFunc_0}(a | s_0) = 0$.

We use the same sequence of reward functions $\RewardFunc_\varepsilon$ as constructed in the proof of Proposition \ref{prop:discontinuity_deterministic}. 
We have $\RewardFunc_\varepsilon \to \RewardFunc_0$ in $(\RewardSpace, \SupNorm{\cdot})$.
From the proof of Proposition \ref{prop:discontinuity_deterministic}, we established that the optimal action set for $\RewardFunc_\varepsilon$ is $A^*(s_0; \RewardFunc_\varepsilon) = \{a_2\}$. Let $\RewardFunc_n=\RewardFunc_{\varepsilon_n}$ for a sequence $\varepsilon_n\to 0$.

According to our policy selection rule, $\OptPolicy_{\RewardFunc_n}(\cdot | s_0)$ is the uniform distribution over $A^*(s_0; \RewardFunc_n) = \{a_2\}$. This means $\OptPolicy_{\RewardFunc_n}(\cdot | s_0)$ is the Dirac measure $\delta_{a_2}$ concentrated at $a_2$. So, $\OptPolicy_{\RewardFunc_n}(a_2 | s_0) = 1$, and $\OptPolicy_{\RewardFunc_n}(a | s_0) = 0$ for all $a \neq a_2$.

We now compute the Total Variation distance $d_{TV}(\OptPolicy_{\RewardFunc_n}(\cdot | s_0), \OptPolicy_{\RewardFunc_0}(\cdot | s_0))$ for $n$ sufficiently large. Let $\mu_n = \OptPolicy_{\RewardFunc_n}(\cdot | s_0) = \delta_{a_2}$ and $\mu_0 = \OptPolicy_{\RewardFunc_0}(\cdot | s_0)$.
The Total Variation distance between two probability measures $\nu_1, \nu_2$ on a countable (or finite, as is the case here for the support of $\mu_0$ and $\mu_n$) set $\ActionSpace'$ is given by $d_{TV}(\nu_1, \nu_2) = \frac{1}{2} \sum_{a \in \ActionSpace'} |\nu_1(a) - \nu_2(a)|$.
Here, $\ActionSpace'$ can be taken as $A^*(s_0; \RewardFunc_0) \cup \{a_2\} = A^*(s_0; \RewardFunc_0)$ since $a_2 \in A^*(s_0; \RewardFunc_0)$.
\begin{align*}
d_{TV}(\mu_n, \mu_0) &= \frac{1}{2} \sum_{a \in A^*(s_0; \RewardFunc_0)} |\mu_n(a) - \mu_0(a)| \\
&= \frac{1}{2} \left( |\mu_n(a_2) - \mu_0(a_2)| + \sum_{a \in A^*(s_0; \RewardFunc_0), a \neq a_2} |\mu_n(a) - \mu_0(a)| \right) \\
&= \frac{1}{2} \left( |1 - 1/m| + \sum_{a \in A^*(s_0; \RewardFunc_0), a \neq a_2} |0 - 1/m| \right).
\end{align*}
There are $m-1$ terms in the sum. So,
\begin{align*}
d_{TV}(\mu_n, \mu_0) = \frac{1}{2} \left( \frac{m-1}{m} + (m-1) \cdot \frac{1}{m} \right) = \frac{1}{2} \frac{2(m-1)}{m} = \frac{m-1}{m}. 
\end{align*}
Since $m \ge 2$, we have $(m-1)/m \ge 1/2$.
Thus, for $n$ sufficiently large, $d_{TV}(\OptPolicy_{\RewardFunc_n}(\cdot | s_0), \OptPolicy_{\RewardFunc_0}(\cdot | s_0)) = (m-1)/m$.
As $(m-1)/m$ does not tend to $0$ as $n \to \infty$ (it is a constant $\ge 1/2$), the sequence of policies $\OptPolicy_{\RewardFunc_n}(\cdot | s_0)$ does not converge to $\OptPolicy_{\RewardFunc_0}(\cdot | s_0)$ in Total Variation distance.
Therefore, the map $\RewardFunc \mapsto \OptPolicy_\RewardFunc(\cdot | s_0)$ is discontinuous at $\RewardFunc_0$.
\end{proof}

\begin{remark}
This analysis highlights that the mapping from rewards to optimal policies exhibits stability (continuity) only under strong structural conditions ensuring the uniqueness of optimal actions. In the absence of such conditions, the mapping is generally unstable, with small perturbations in the reward function potentially leading to abrupt changes in the optimal policy. This has significant implications for the robustness and practical implementation of reinforcement learning algorithms.
\end{remark}

\section{RL for LLMs With a Single Reward Model}\label{sec-llm}

The theoretical framework for the continuity of the reward-policy map has significant implications for the training and behavior of large language models (LLMs), especially when fine-tuned using reinforcement learning (RL) such as Reinforcement Learning from Human Feedback (RLHF) (see e.g., \cite{ouyang2022training,bai2022constitutional,lee2023rlaif}) or Reinforcement Learning with Verifiable Rewards (RLVR) (see e.g., \cite{guo2025deepseek,su2025crossing}). In this section, we frame LLM text generation as an MDP and discuss how the continuity (or lack thereof) of the optimal policy with respect to the reward function can explain certain observed behaviors and challenges in LLM alignment.

\subsection{Framing LLM Text Generation as an MDP}

We model the sequential text generation process of an LLM as an infinite-horizon discounted Markov Decision Process (MDP) defined by the tuple $(\StateSpace, \ActionSpace, \TransKernel, \RewardFunc, \gamma)$. A rigorous analysis requires establishing a topology on the state and action spaces.

The action space $\ActionSpace$ is the model's finite vocabulary, and the state space $\StateSpace$ comprises the finite set of all possible token sequences up to a maximum length $T_{\max}$. A fundamental property of any finite set endowed with a metric is that the resulting metric space is always compact. Therefore, the assumption that $\ActionSpace$ and $\StateSpace$ are compact metric spaces is automatically satisfied for any choice of metric. This assumption, while a necessary precondition for the general theory, thus poses no practical constraint on the LLM setting.

A more subtle but crucial point arises from the topology of these finite spaces. Any metric on a finite set induces the discrete topology, in which every subset is an open set. A direct mathematical consequence of this is that any function from a space with the discrete topology to any other metric space is necessarily continuous. Thus, the assumption that the reward function $\RewardFunc(s, a)$ is continuous on the product space $(\StateSpace \times \ActionSpace, d)$ is also automatically satisfied for any reward function. This means that in the finite LLM setting, this assumption is not a restrictive simplification of reality but a direct consequence of the problem's structure.

The remaining components of the MDP are then defined. The transition kernel $\TransKernel$ is deterministic, as the next state is formed by concatenating the current state with the chosen action token. Finally, the policy $\Policy(a|s)$ is embodied by the LLM, and the goal of reinforcement learning is to find an optimal policy $\OptPolicy_\RewardFunc$ for a given reward $\RewardFunc$.

This precise framing reveals that the source of policy instability does not lie in a potential failure of continuity of the reward function on its domain. Instead, the analytical framework's power hinges on \textit{how perturbations in the reward function's values} propagate through the Bellman operator to the Q-function, and critically, how the non-continuous nature of the $\Argmax$ operator acts upon this Q-function. This formulation, therefore, shifts the focus from the topological properties of the state-action space, which are trivially satisfied, to the functional properties of the Bellman and $\Argmax$ operators. This provides a rigorous foundation for analyzing policy stability not as a failure of continuity on the state space, but as a consequence of the inherent discontinuity of optimization over a discrete action set.

\subsection{Theoretical Implications of the Reward-Policy Map}

The stability results for the underlying value functions provide a foundation for our analysis. Proposition~\ref{prop:q_continuity} establishes that the optimal action-value function $\OptQFunc_\RewardFunc$ is Lipschitz continuous with respect to the reward function $\RewardFunc$. This suggests a degree of robustness at the value level: small changes to the reward model lead to proportionally small and controlled changes in the optimal Q-values. Furthermore, Lemma~\ref{lemma:argmax_uhc} shows that the set of optimal actions $A^*(s; \RewardFunc)$ is an upper hemi-continuous (u.h.c.) correspondence. This crucial property provides a form of "outer" stability: it guarantees that the set of best next tokens will not suddenly expand to include actions that were previously far from optimal. It also ensures that any convergent sequence of optimal actions (for a converging sequence of reward functions) will have its limit within the new set of optimal actions. However, this provides no guarantee that every action that was originally optimal will remain part of the optimal set after a small perturbation of the reward function. A critical consequence of a correspondence being only upper hemi-continuous is that it permits the set of optimal actions $A^*(s; \RewardFunc)$ to \textit{abruptly shrink}. In other words, a slight change in $\RewardFunc$ can cause previously optimal actions to "disappear" by becoming strictly suboptimal. This lack of preservation, which corresponds to a failure of lower hemi-continuity (l.h.c.), is a primary driver of the policy instabilities central to our analysis.

This is precisely the core issue: while the value function enjoys robust stability guarantees, these guarantees do not propagate to the policy level. The potential for the optimal action set to abruptly shrink means that any policy, which is fundamentally a selection from this set, is inherently fragile. Therefore, the stability of the policy itself becomes a far more delicate and critical property for predictable behavior. The continuity of the reward-policy map $M_{RL}: \RewardFunc \mapsto \OptPolicy_\RewardFunc$ depends heavily on the structure of the optimal action set.

\paragraph{Conditions for Policy Continuity.}
As established in Theorem~\ref{thm:continuity_rigorous}, if the optimal action is unique for a given reward function $\RewardFunc_0$ and all reward functions in its vicinity, then the mapping to this deterministic optimal policy is continuous. In an LLM context, this stable regime would mean that small adjustments to the reward model only lead to minor, predictable changes in text generation.

\paragraph{Policy Discontinuities due to Non-Unique Optimal Actions.}
The uniqueness condition is strong and often violated in language generation, where multiple words or phrases can be equally valid continuations. When multiple optimal actions exist, the policy map becomes prone to discontinuities, as shown in Propositions~\ref{prop:discontinuity_deterministic} and \ref{prop:discontinuity_stochastic}. A slight perturbation to the reward function can act as a tie-breaker, causing a deterministic policy to abruptly switch its choice of action or a stochastic policy to drastically shift its probability mass. This theoretical instability provides a formal basis for the brittle behaviors observed in practice, such as sudden changes in style or safety profile in response to minor changes in the reward model or prompt.

\subsection{Analysis of Alignment and Reasoning Phenomena}

The theoretical framework can be directly applied to formalize and explain specific, observable behaviors in aligned or reasoning LLMs. In particular, we now use our results to analyze two fundamental challenges when using RL for LLM training. First, we will examine phenomena arising from \textit{incomplete reward specifications}, where we formalize how a policy can be perfectly rational for its given objective yet suboptimal for the true, intended goal. Second, we will analyze behaviors stemming from the \textit{degeneracy of optima}, where multiple distinct policies are equally optimal, and show how introducing an additive reward component functions as a \textit{tie-breaker} to resolve this issue and enforce more specific behaviors.

\subsubsection{The "Clever Slacker": Suboptimality from Incomplete Rewards}

A common challenge in alignment and reasoning is the "clever slacker" phenomenon, where an LLM produces factually correct responses that nevertheless ignore user-specified constraints or instructions, or generates answers via deceptive shortcuts. This can be modeled as the agent optimizing an incomplete reward function. The following proposition proves that such a policy is strictly suboptimal under the complete, desired reward objective.

\begin{proposition}[Suboptimality from Incomplete Rewards (General Form)] \label{prop:incomplete_reward_general}
Let $\RewardFunc_{train} \in \RewardSpace$ be the training reward function and $\RewardFunc_{missing} \in \RewardSpace$ be the missing reward component. The true reward is $\RewardFunc_{true} = \RewardFunc_{train} + \RewardFunc_{missing}$. Let $\pi^*_{train} \in \Pi^*(\RewardFunc_{train})$ be any optimal policy for $\RewardFunc_{train}$, and let $\mu$ be the initial state distribution. If there exists a state $s_0 \in \StateSpace$ and an action $a_2 \in \ActionSpace$ satisfying:
\begin{enumerate}
    \item \textbf{Action Optimality}: $a_2$ is optimal under the training reward: $a_2 \in A^*(s_0; \RewardFunc_{train})$.
    \item \textbf{Positive Advantage for Missing Reward}: The action $a_2$ has a strictly positive advantage under the missing reward component when evaluated with the policy $\pi^*_{train}$:
    $$A^{\pi^*_{train}}_{\RewardFunc_{missing}}(s_0, a_2) > 0,$$
    where the advantage is defined as $A^{\pi}_{R}(s, a) := Q^{\pi}_{R}(s, a) - V^{\pi}_{R}(s)$.
    \item \textbf{State Reachability}: The state $s_0$ has a non-zero probability of being visited at some step when starting from the initial state distribution $\mu$ and following the policy $\pi^*_{train}$.
\end{enumerate}
Then, the policy $\pi^*_{train}$ is strictly suboptimal for the true reward function $\RewardFunc_{true}$.
\end{proposition}
\begin{proof}
To prove that $\pi^*_{train}$ is strictly suboptimal for $\RewardFunc_{true}$, we invoke the Policy Improvement Theorem (see e.g., \cite{sutton1998reinforcement}). We show there exists an action whose value is strictly greater than the state-value achieved by the policy, which is a sufficient condition for suboptimality. We must show $Q^{\pi^*_{train}}_{\RewardFunc_{true}}(s_0, a_2) > V^{\pi^*_{train}}_{\RewardFunc_{true}}(s_0)$.

Let's expand the advantage condition (Condition 2):
$Q^{\pi^*_{train}}_{\RewardFunc_{missing}}(s_0, a_2) - V^{\pi^*_{train}}_{\RewardFunc_{missing}}(s_0) > 0$.
Using the linearity of the value function for a fixed policy ($V^{\pi}_{R_1+R_2} = V^{\pi}_{R_1} + V^{\pi}_{R_2}$), we can write:
$V^{\pi^*_{train}}_{\RewardFunc_{true}}(s_0) = V^{\pi^*_{train}}_{\RewardFunc_{train}}(s_0) + V^{\pi^*_{train}}_{\RewardFunc_{missing}}(s_0)$.
Since $\pi^*_{train}$ is optimal for $\RewardFunc_{train}$, $V^{\pi^*_{train}}_{\RewardFunc_{train}}(s_0) = V^*_{train}(s_0)$.
So, $V^{\pi^*_{train}}_{\RewardFunc_{true}}(s_0) = V^*_{train}(s_0) + V^{\pi^*_{train}}_{\RewardFunc_{missing}}(s_0)$.

Now let's analyze the Q-value of taking action $a_2$:
$Q^{\pi^*_{train}}_{\RewardFunc_{true}}(s_0, a_2) = Q^{\pi^*_{train}}_{\RewardFunc_{train}}(s_0, a_2) + Q^{\pi^*_{train}}_{\RewardFunc_{missing}}(s_0, a_2)$.
Since $\pi^*_{train}$ is optimal for $\RewardFunc_{train}$, $Q^{\pi^*_{train}}_{\RewardFunc_{train}} = \OptQFunc_{train}$. By Condition 1, $a_2 \in A^*(s_0; \RewardFunc_{train})$, so $\OptQFunc_{train}(s_0, a_2) = V^*_{train}(s_0)$.
Thus, $Q^{\pi^*_{train}}_{\RewardFunc_{true}}(s_0, a_2) = V^*_{train}(s_0) + Q^{\pi^*_{train}}_{\RewardFunc_{missing}}(s_0, a_2)$.

We can now directly compare $Q^{\pi^*_{train}}_{\RewardFunc_{true}}(s_0, a_2)$ with $V^{\pi^*_{train}}_{\RewardFunc_{true}}(s_0)$. The inequality $Q^{\pi^*_{train}}_{\RewardFunc_{true}}(s_0, a_2) > V^{\pi^*_{train}}_{\RewardFunc_{true}}(s_0)$ holds if and only if:
$$ V^*_{train}(s_0) + Q^{\pi^*_{train}}_{\RewardFunc_{missing}}(s_0, a_2) > V^*_{train}(s_0) + V^{\pi^*_{train}}_{\RewardFunc_{missing}}(s_0). $$
This simplifies to $Q^{\pi^*_{train}}_{\RewardFunc_{missing}}(s_0, a_2) > V^{\pi^*_{train}}_{\RewardFunc_{missing}}(s_0)$, which is exactly Condition 2 in its expanded form. Since the condition holds by hypothesis, this implies that $\pi^*_{train}$ is not an optimal policy for $\RewardFunc_{true}$.

Furthermore, because Condition 3 ensures that state $s_0$ is reachable from the initial state distribution $\mu$, this local improvement opportunity at $s_0$ will lead to a strict increase in the overall expected return from the start states. Therefore, the policy $\pi^*_{train}$ is strictly suboptimal for the true reward function $\RewardFunc_{true}$.
\end{proof}

This formal result provides a unified lens to explain several critical alignment failures, which can be understood as policies that are optimal for an incomplete training objective but suboptimal for the true, desired objective. We analyze two such key phenomena below.

\paragraph{Instruction Following Failure.}
A primary application of our proposition is in explaining why models often fail to adhere to user-specified constraints, a behavior sometimes termed the "clever slacker" phenomenon. This occurs when the training reward, $\RewardFunc_{\text{train}}$, primarily measures a core task requirement, such as factual correctness, while ignoring secondary instructions like output format, style, or negative constraints. The adherence to these instructions represents the missing reward component, $\RewardFunc_{\text{missing}}$. If a model can produce a factually correct answer both with and without following the constraints, both paths may be equally optimal under $\RewardFunc_{\text{train}}$. The policy is therefore not being "disobedient"; it is perfectly rational in choosing the path of least resistance to maximize the objective it was given. This behavior is precisely captured by our proposition, which demonstrates that such a policy is strictly suboptimal under the true, complete reward function $\RewardFunc_{\text{true}} = \RewardFunc_{\text{train}} + \RewardFunc_{\text{missing}}$.

\paragraph{Spurious Reasoning.}
A more subtle and complex failure mode is \textit{spurious reasoning}, where the model produces a correct final answer preceded by a chain-of-thought that is logically flawed or not causally linked to the result. This is also a consequence of an incomplete Outcome-Based Reward (OBR). Here, the training reward $\RewardFunc_{\text{train}} = \RewardFunc_{\text{outcome}}$ only values the correctness of the final answer. The crucial missing component, $\RewardFunc_{\text{missing}} = \RewardFunc_{\text{process}}$, should reward the logical validity and faithfulness of the reasoning process itself. Because any reasoning path leading to the correct outcome is equally valued, the model may discover a low-effort strategy: first retrieve or guess the answer, and then generate a syntactically plausible but non-causal justification post-hoc. This policy, which fabricates a reasoning process, is another manifestation of a "clever slacker". While it perfectly optimizes the outcome-based objective, it is demonstrably suboptimal under the true objective that values faithful reasoning, as explained by our formal result.

\subsubsection{The Tie-Breaker Effect: Resolving Degeneracy with Additive Rewards}

A key challenge in reward design is the \textit{degeneracy of optimal policies}, which arises when a primary objective function, such as accuracy, deems multiple, behaviorally distinct policies to be equally optimal. This allows for undesirable behaviors, such as stylistic inconsistency or inefficient reasoning, to emerge. The practice of introducing an additional reward component is a form of \textit{reward engineering} designed to break this degeneracy. The following proposition formalizes how such an additive reward can function as a tie-breaker to enforce a specific, desired policy.

\begin{proposition}[The Tie-Breaker Effect]\label{prop:tie-breaker}
Let Assumptions \ref{assump:spaces}-\ref{assump:discount} hold. Let $\RewardFunc_0 \in \RewardSpace$ be a reward function, and suppose for some state $s_0 \in \StateSpace$, there exist at least two distinct optimal actions, $a_1, a_2 \in A^*(s_0; \RewardFunc_0)$. Then for any $\varepsilon>0$, there exists a perturbed reward function $\RewardFunc' \in \RewardSpace$ such that $\|\RewardFunc' - \RewardFunc_0\|_{\infty} \leq \varepsilon$ and $Q^*_{\RewardFunc'}(s_0, a_2) > Q^*_{\RewardFunc'}(s_0, a_1)$, i.e.\ \(a_2\) becomes strictly optimal over \(a_1\).
\end{proposition}
\begin{proof}
Let \(Q_0=Q^*_{\RewardFunc_0}\).  Since \(a_1,a_2\) are both optimal at \(s_0\),
\(Q_0(s_0,a_1)=Q_0(s_0,a_2)\).

\noindent\textbf{1. Bump in \(Q\)-space.}
Choose a continuous bump as Proposition~\ref{prop:discontinuity_deterministic} that 
\[
\varphi\in C(\StateSpace\times\ActionSpace),\quad
0\le\varphi\le1,
\quad
\varphi(s_0,a_2)=1,
\]
and \(\varphi(s_0,a_j)=0\) for any other \(a_j\in A^*(s_0;\RewardFunc_0)\).  Define
\[
Q_{\varepsilon'}(s,a)=Q_0(s,a)+\varepsilon'\,\varphi(s,a),
\]
where \(\varepsilon'=\varepsilon/(1+\gamma)\), so \(\|Q_{\varepsilon'}-Q_0\|_\infty=\varepsilon'\to0\).

\noindent\textbf{2. Invert the Bellman operator.}
For any continuous \(Q\), define
\[
R_Q(s,a)
=Q(s,a)
-\gamma\!\int_{\StateSpace}\max_{a'}Q(s',a')\,P(ds'\!\mid s,a).
\]
Then one checks \(\BellmanOp_{R_Q}(Q)=Q\), and by contraction \(Q\) is the
unique \(Q^*_{R_Q}\).  Hence for each \(\varepsilon>0\), let 
\[
\RewardFunc_\varepsilon=R_{\,Q_0+\varepsilon'\varphi}
\]
so that
\[
Q^*_{\RewardFunc_\varepsilon}
=Q_0+\varepsilon'\,\varphi.
\]

\noindent\textbf{3. Convergence of rewards.}
Since \(R_Q\) depends continuously (in sup-norm) on \(Q\) and
\(\|Q_0+\varepsilon'\varphi-Q_0\|_\infty=\varepsilon'\), we have
\(\|\RewardFunc_\varepsilon-\RewardFunc_0\|_\infty \le \varepsilon'(1+\gamma)= \varepsilon \to 0\).

\noindent\textbf{4. Tie-breaking at \(s_0\).}
Under \(Q^*_{\RewardFunc_\varepsilon}=Q_0+\varepsilon'\varphi\),
\[
Q^*_{\RewardFunc_\varepsilon}(s_0,a_2)
-
Q^*_{\RewardFunc_\varepsilon}(s_0,a_1)
=
\bigl[Q_0(s_0,a_2)-Q_0(s_0,a_1)\bigr]
+\varepsilon'\bigl[\varphi(s_0,a_2)-\varphi(s_0,a_1)\bigr]
=0+\varepsilon'>0.
\]
Thus \(a_2\) is strictly preferred over \(a_1\). This completes the proof.
\end{proof}

The "tie-breaker" effect formalized in Proposition~\ref{prop:tie-breaker} is a powerful tool for reward engineering, explaining how additive rewards $\Delta\RewardFunc_0=\RewardFunc' - \RewardFunc_0$ can resolve the degeneracy of optima and enforce specific desired behaviors. We analyze two key applications below: enforcing format adherence and promoting efficient reasoning.

\paragraph{Enforcing Format and Style Consistency.}
Consider a scenario where multiple responses are equally valid under a base reward function, $\RewardFunc_0$, which primarily measures correctness. For example, an answer could be given in free-form text or as a structured JSON object. Both might be equally optimal, leading to a large optimal policy set $\optPolicySet(\RewardFunc_0)$. To enforce a specific format, one can introduce a small, additive "bonus" reward, $\Delta\RewardFunc_0=\RewardFunc' - \RewardFunc_0$, for responses that adhere to the desired structure. This bonus acts as a tie-breaker. According to our proposition, this additive reward makes the policy that generates the correctly formatted output strictly optimal. This can cause the model's behavior to "snap" discontinuously to the preferred style, rather than shifting gradually, which explains why such changes can be abrupt.

\paragraph{Promoting Efficient Reasoning.}
The tendency for models to generate unnecessarily verbose responses is another consequence of a degenerate optimal policy set under an accuracy-only reward, $\RewardFunc_0 = \RewardFunc_{\text{acc}}$. Both short and long reasoning paths that reach the correct answer are equally optimal. To solve this, one can introduce a length penalty, which is a \textit{negative} additive reward. This penalty also functions as a powerful tie-breaker. Among all paths that achieve the same accuracy reward, the most efficient path will now have the strictly highest value, as it incurs the lowest total penalty. This modification collapses the optimal policy set $\optPolicySet(\RewardFunc_0 + \Delta\RewardFunc_0)$ to contain only the most concise policies. By making the efficient policy uniquely optimal, this approach clarifies the learning signal and promotes a more efficient reasoning process.

\subsection{Broader Implications and Mitigation Strategies}

The analysis throughout this work highlights a fundamental distinction between applying reinforcement learning to LLMs versus traditional domains. In tasks with objective, well-defined rewards like the game of Go, achieving the optimal value $\ValueFunc^*$ is a sufficient goal; the policy is merely a means to an end. This paradigm shifts fundamentally for LLMs. The policy's output, i.e. the sequence of generated tokens, is the final product consumed by the user, and the reward function is not a ground-truth oracle but an imperfect proxy for nuanced human preferences (RLHF) or a rule-based function (RLVR). Consequently, the policy $\Policy$ itself must become a primary object of analysis. \textit{The continuity of the reward-policy map thus emerges as a critical diagnostic for the robustness and trustworthiness of the aligned model. In essence, for language models, the policy's behavior is not merely a means to an end; it is, in large part, the end itself.}

Given this critical role of the policy, the potential for discontinuities, as established by our theoretical results, underscores a central challenge in LLM alignment. The brittleness often observed in model behavior can be seen as a direct consequence of these instabilities, which arise when multiple generation strategies are near-optimal under a flawed or incomplete reward model. This highlights the importance of carefully engineering the learning objective to mitigate these issues. A primary approach is to introduce regularization techniques that encourage policy smoothness and resolve the degeneracy of optima.

In practice, the most prevalent form of such regularization is fine-tuning a pre-trained base model, $\basePolicy$, with a KL-divergence penalty. The objective is thus not merely to maximize the reward $\RewardFunc$, but to find a policy $\pi_{\text{RL}}$ that balances this with fidelity to the base model, governed by the objective:
\begin{equation}
\begin{aligned}
J(\Policy) &=  \mathbb{E}_{\tau \sim \Policy} \left[ \sum_{t=0}^\infty \gamma^t \left( \RewardFunc(s_t, a_t) - \beta\mathcal{D}_{\mathrm{KL}}(\Policy(\cdot|s_t) \| \pi_{\text{base}}(\cdot|s_t)) \right) \right] \\
& = \mathbb{E}_{\tau \sim \Policy} \left[ \sum_{t=0}^\infty \gamma^t \left( \RewardFunc(s_t, a_t) - \beta\log\frac{\Policy(a_t|s_t)}{\pi_{\text{base}}(a_t|s_t)} \right) \right].
\end{aligned}
\end{equation}
This KL-divergence term serves a dual purpose. It preserves the vast linguistic knowledge of $\basePolicy$, preventing catastrophic forgetting. Furthermore, as our analysis has shown, it acts as a crucial \textit{tie-breaker} when the optimal policy set $\optPolicySet(\RewardFunc)$ is large, guiding the algorithm towards the optimal policy that is closest to the base model's inherent capabilities and stylistic biases. This provides a practical, albeit constrained, mechanism for managing the degeneracy of the optima that we have identified. Alongside regularizing towards a base policy, another powerful mitigation strategy is to directly promote policy stochasticity via \textit{entropy regularization}, which we will formalize in Section~\ref{subsec:Entropy_Regularization}.

\section{LLMs Trained with Multiple Specialized Reward Models}\label{sec:llm-multi-reward}

Large Language Models (LLMs) are increasingly trained to handle diverse tasks and exhibit nuanced behaviors. A sophisticated approach to achieve this involves using multiple, specialized reward models, each tailored to a specific class of data or desired capability. This section extends our continuity analysis to such a multi-reward RL training paradigm across diverse domains (see e.g., \cite{liang2025modomodo,cheng2025revisiting,su2025crossing}).

\subsection{Framework: Single LLM, Multiple Data Classes and Reward Models}

Consider a scenario where a single LLM policy, $\Policy: \StateSpace \to \PolicySpace(\ActionSpace)$, is trained to perform well across $N$ distinct classes of data or tasks. Let these classes be denoted by $k \in \{1, \ldots, N\}$. The framework is shown in Figure~\ref{fig:multi-reward-framework}.
\begin{itemize}
    \item \textbf{Data Classes ($\mathcal{D}_k$):} Each class $\mathcal{D}_k$ represents a distribution of initial states (e.g., prompts) specific to a certain domain or task (e.g., mathematical reasoning, coding, question answering, safety alignment).
    \item \textbf{Specialized Reward Models ($\RewardFunc_k$):} For each data class $\mathcal{D}_k$, there is a corresponding reward model that provides a reward function $\RewardFunc_k: \StateSpace \times \ActionSpace \to \R$. This function $\RewardFunc_k(s,a)$ evaluates the quality of action $a$ (generating a token) in state $s$ (current sequence) specifically for task $k$. We assume each $\RewardFunc_k \in \RewardSpace = \Cont(\StateSpace \times \ActionSpace)$. Let $\boldsymbol{\RewardFunc} = (\RewardFunc_1, \ldots, \RewardFunc_N)$ denote the tuple of these reward functions. The space of such tuples is $\RewardSpace^N$, equipped with the norm ${\SupNorm{\boldsymbol{\RewardFunc} - \boldsymbol{\RewardFunc}'}}_{N} = \max_{1 \le k \le N} \SupNorm{\RewardFunc_k - \RewardFunc_k'}$.
    \item \textbf{Training Objective:} The LLM policy $\Policy$ is trained to optimize a global objective that aggregates performance across all $N$ task-reward pairs. This is often formulated as maximizing the weighted sum of expected discounted rewards. During training, if an initial state $s_0$ is sampled from $\mathcal{D}_k$ (with probability $p_k$, where $\sum_{k=1}^N p_k = 1$ and we assume $p_k > 0$ for all $k$), then the subsequent rewards for the trajectory generated by $\Policy$ are drawn from $\RewardFunc_k$. The overall objective is to find a policy $\OptPolicy_{\boldsymbol{\RewardFunc}}$ that maximizes:
    \begin{equation} \label{eq:multi_objective_revised}
        J\left(\Policy; \boldsymbol{\RewardFunc}, \{\mathcal{D}_k\}, \{p_k\}\right) = \sum_{k=1}^N p_k \mathbb{E}_{s_0 \sim \mathcal{D}_k, \tau \sim \Policy(\cdot|s_0)} \left[ \sum_{t=0}^\infty \gamma^t \RewardFunc_k(s_t, a_t) \right].
    \end{equation}
    Here, $\tau \sim \Policy(\cdot|s_0)$ denotes a trajectory generated by policy $\Policy$ starting from $s_0$.
\end{itemize}
The state space $\StateSpace$, action space $\ActionSpace$, transition kernel $\TransKernel$, and discount factor $\gamma$ are defined as in Section \ref{sec-rl} (satisfying Assumptions \ref{assump:spaces}, \ref{assump:transition}, and \ref{assump:discount}). The single policy $\Policy$ must learn to adapt its behavior based on the implicit context of the input state $s$, even though it may not explicitly receive the class index $k$ as input.

\begin{figure}[h!]
  \centering
  \begin{tikzpicture}[
    font=\sffamily,
    node distance=0.7cm and 1.5cm,
    >=Latex,
    every node/.style={align=center, rounded corners=3pt, thick},
    data/.style={draw, fill=gray!15, minimum width=2cm, minimum height=1cm, font=\small},
    policy/.style={draw, fill=blue!15, minimum width=2.5cm, minimum height=6cm, font=\normalsize},
    reward/.style={draw, fill=orange!15, minimum width=2cm, minimum height=1cm, font=\small},
    obj/.style={draw, fill=red!15, minimum width=2cm, minimum height=1cm, font=\small}
  ]

    \node[policy] (LLM) at (0,-0.75) {\textbf{LLM}};

    \node[data,left=of LLM,yshift=2.5cm] (D1) {$\mathcal{D}_1$};
    \node[data,left=of LLM] (D2) {$\mathcal{D}_2$};
    \node[below=0.35cm of D2,font=\normalsize] (dotsD) {$\vdots$};
    \node[data,below=0.35cm of dotsD] (DN) {$\mathcal{D}_N$};

    \node[reward,right=of LLM,yshift=2.5cm] (R1) {$\RewardFunc_1$};
    \node[reward,right=of LLM] (R2) {$\RewardFunc_2$};
    \node[below=0.35cm of R2,font=\normalsize] (dotsR) {$\vdots$};
    \node[reward,below=0.35cm of dotsR] (RN) {$\RewardFunc_N$};

    \node[obj,right=1.5cm of R2] (J) {$J(\pi)$};



    \draw[->,thick] (D1.east) -- ([yshift=2.5cm]LLM.west)
    node[midway,above,sloped,font=\small] {prompts};
    \draw[->,thick] (D2.east) -- (LLM.west)
    node[midway,above,sloped,font=\small] {prompts};
    \draw[->,thick] (DN.east) -- ([yshift=-2.5cm]LLM.west)
    node[midway,above,sloped,font=\small] {prompts};

    \draw[->,thick] ([yshift=2.5cm]LLM.east) -- (R1.west) node[midway,above,sloped,font=\small] {rollout};
    \draw[->,thick] (LLM.east) -- (R2.west) node[midway,above,sloped,font=\small] {rollout};
    \draw[->,thick] ([yshift=-2.5cm]LLM.east) -- (RN.west) node[midway,above,sloped,font=\small] {rollout};

    \foreach \r in {R1,R2,RN}{
      \draw[->,thick] (\r.east)--(J.west);
    }
    \node[font=\small,above] at ($(R2.east)!0.5!(J.west)$) {aggregate};

    \draw[->,thick] (J.south) |- ++(0,-3.0cm) -| (LLM.south) node[below=0.45cm of RN,font=\small]{backpropagation \& update policy};

  \end{tikzpicture}
  \vspace{0.2cm}
  \caption{The framework for training a single LLM policy with multiple data classes and specialized reward models. Prompts from data classes $\mathcal{D}_k$ are used to generate trajectories (rollouts) with the policy $\pi$. Each trajectory is evaluated by its corresponding reward model $\RewardFunc_k$. The rewards are aggregated into a global objective function $J(\pi)$, which is then used to update the policy via backpropagation.}
  \label{fig:multi-reward-framework}
\end{figure}
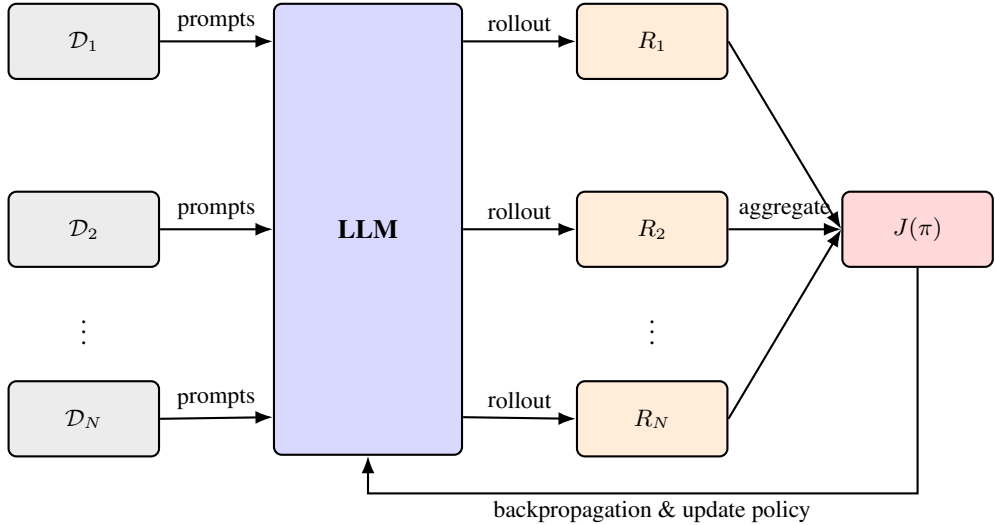

\subsection{Analyzing Policies Derived from State-Dependent Effective Rewards}

The global objective $J(\Policy)$ in Eq.~\eqref{eq:multi_objective_revised} averages performance across episodes that are governed by distinct, episode-specific reward functions $\RewardFunc_k$. A single policy $\Policy(a|s)$, lacking explicit knowledge of $k$, must infer the context from $s$. The direct characterization of $\OptPolicy_{\boldsymbol{\RewardFunc}}$ maximizing $J(\Policy)$ via a simple Bellman optimality equation (as in Section \ref{sec-rl}) is generally not available.

To apply the continuity analysis tools, we focus on a structured model of how an LLM might address this multi-task challenge: by forming an internal, state-dependent \emph{effective reward function}, $\RewardFunc_{eff}$. This function aggregates the specialized rewards based on the current context:
\begin{equation} \label{eq:effective_reward}
    \RewardFunc_{eff}(s,a; \boldsymbol{\RewardFunc}) = \sum_{k=1}^N w_k(s) \RewardFunc_k(s,a).
\end{equation}
Here, $w_k: \StateSpace \to [0,1]$ are continuous weighting functions satisfying $\sum_{k=1}^N w_k(s) = 1$ for all $s \in \StateSpace$. These weights model the LLM's assessment of the relevance of task $k$ given state $s$. The following analysis assumes that the weighting functions $w_k(s)$ are given. Crucially, these $w_k(s)$ are assumed to be fixed and do not change as $\boldsymbol{\RewardFunc}$ is perturbed. If $w_k(s)$ were themselves functions of $\boldsymbol{\RewardFunc}$, a more comprehensive stability analysis would be required.

The policy subject to our continuity analysis, denoted ${\OptPolicy_{\boldsymbol{\RewardFunc}}}^{eff}$, is the optimal policy for the standard MDP defined by this $\RewardFunc_{eff}$. It is thus greedy with respect to the optimal action-value function $\OptQFunc_{eff}(s,a; \boldsymbol{\RewardFunc})$, which is the unique fixed point of the Bellman optimality equation:
\begin{equation} \label{eq:bellman_effective}
(\BellmanOp_{eff, \boldsymbol{\RewardFunc}} Q)(s, a) = \RewardFunc_{eff}(s,a; \boldsymbol{\RewardFunc}) + \gamma \int_{\StateSpace} \max_{a' \in \ActionSpace} Q(s', a') \TransKernel(ds' | s, a).
\end{equation}
The optimal policy ${\OptPolicy_{\boldsymbol{\RewardFunc}}}^{eff}$ has its support contained within $A_{eff}^*(s; \boldsymbol{\RewardFunc}) = \Argmax_{a \in \ActionSpace} \OptQFunc_{eff}(s, a; \boldsymbol{\RewardFunc})$.

\begin{remark}[On the Relationship between ${\OptPolicy_{\boldsymbol{\RewardFunc}}}^{eff}$ and $J(\Policy)$]
The policy ${\OptPolicy_{\boldsymbol{\RewardFunc}}}^{eff}$ analyzed here is optimal for the constructed $\RewardFunc_{eff}$. The alignment of ${\OptPolicy_{\boldsymbol{\RewardFunc}}}^{eff}$ with $\OptPolicy_{\boldsymbol{\RewardFunc}}$ (the maximizer of $J(\Policy)$) depends on how well $w_k(s)$ are chosen or learned. If, for instance, the initial state distributions were identical ($\mathcal{D}_k = \mathcal{D}_{init}$ for all $k$) and one chose $w_k(s) = p_k$ (constant weights), then $J(\Policy)$ simplifies to $\mathbb{E}_{s_0 \sim \mathcal{D}_{init}} [V_{\RewardFunc_{eff}}^{\Policy}(s_0)]$, and ${\OptPolicy_{\boldsymbol{\RewardFunc}}}^{eff}$ would indeed be $\OptPolicy_{\boldsymbol{\RewardFunc}}$.
Our analysis in this section focuses on the stability of ${\OptPolicy_{\boldsymbol{\RewardFunc}}}^{eff}$ resulting from such an effective reward mechanism.
\end{remark}

\begin{lemma}[Properties of the Effective Reward Mapping]\label{lemma:effective_reward_properties}
Let each $\RewardFunc_k \in \RewardSpace$, and let $w_k: \StateSpace \to [0,1]$ be continuous functions such that $\sum_{k=1}^N w_k(s) = 1$ for all $s \in \StateSpace$.
\begin{enumerate}
    \item For any $\boldsymbol{\RewardFunc} \in \RewardSpace^N$, the effective reward function $\RewardFunc_{eff}(\cdot, \cdot; \boldsymbol{\RewardFunc})$ defined in Eq.~\eqref{eq:effective_reward} is continuous on $\StateSpace \times \ActionSpace$.
    \item The mapping $\boldsymbol{\RewardFunc} \mapsto \RewardFunc_{eff}(\cdot, \cdot; \boldsymbol{\RewardFunc})$ is Lipschitz continuous from $(\RewardSpace^N, {\SupNorm{\cdot}}_{N})$ to $(\RewardSpace, \SupNorm{\cdot})$ with Lipschitz constant 1:
    \[ \SupNorm{\RewardFunc_{eff}(\cdot,\cdot;\boldsymbol{\RewardFunc}) - \RewardFunc_{eff}(\cdot,\cdot;\boldsymbol{\RewardFunc}')} \le {\SupNorm{\boldsymbol{\RewardFunc} - \boldsymbol{\RewardFunc}'}}_{N}. \]
\end{enumerate}
\end{lemma}
\begin{proof}
Since each $\RewardFunc_k$ is continuous on $\StateSpace \times \ActionSpace$ and each $w_k$ is continuous on $\StateSpace$, their products $w_k(s)\RewardFunc_k(s,a)$ are continuous. The finite sum $\RewardFunc_{eff}(s,a; \boldsymbol{\RewardFunc}) = \sum_{k=1}^N w_k(s) \RewardFunc_k(s,a)$ is therefore continuous. 
For any $(s,a) \in \StateSpace \times \ActionSpace$, 
\begin{align*}
|\RewardFunc_{eff}(s,a;\boldsymbol{\RewardFunc}) - \RewardFunc_{eff}(s,a;\boldsymbol{\RewardFunc}')| &= \left| \sum_{k=1}^N w_k(s) (\RewardFunc_k(s,a) - \RewardFunc_k'(s,a)) \right| \\
&\le \sum_{k=1}^N w_k(s) |\RewardFunc_k(s,a) - \RewardFunc_k'(s,a)| \\
&\le \sum_{k=1}^N w_k(s) \SupNorm{\RewardFunc_k - \RewardFunc_k'} \\
&\le \left( \sum_{k=1}^N w_k(s) \right) \max_{1 \le j \le N} \SupNorm{\RewardFunc_j - \RewardFunc_j'} \\
&= 1 \cdot {\SupNorm{\boldsymbol{\RewardFunc} - \boldsymbol{\RewardFunc}'}}_{N}.
\end{align*}
Taking the supremum over $(s,a)$ on the left-hand side yields the result.
\end{proof}

\subsection{Continuity of the Policy Map for the Effective Reward Model}\label{subsec:cont_eff_reward}

We now analyze the continuity of ${\OptPolicy_{\boldsymbol{\RewardFunc}}}^{eff}$ with respect to $\boldsymbol{\RewardFunc}$, based on the properties of $\RewardFunc_{eff}$.

\begin{proposition}[Lipschitz Stability of Effective Q-Function]\label{prop:lipschitz-effective-q-revised}
Under Assumptions \ref{assump:spaces}-\ref{assump:discount} and the existence of continuous $w_k(s)$ as defined above, the mapping $\boldsymbol{\RewardFunc} \mapsto \OptQFunc_{eff}(\cdot,\cdot;\boldsymbol{\RewardFunc})$ from $(\RewardSpace^N, {\SupNorm{\cdot}}_{N})$ to $(\Cont(\StateSpace \times \ActionSpace), \SupNorm{\cdot})$ is Lipschitz continuous with constant $1/(1-\gamma)$.
\end{proposition}
\begin{proof}
This follows from Proposition \ref{prop:q_continuity} and Lemma \ref{lemma:effective_reward_properties}.
\end{proof}

\begin{lemma}[Upper Hemi-Continuity of Effective Argmax Correspondence]\label{lemma:effective_argmax_uhc-revised}
Under Assumptions \ref{assump:spaces}-\ref{assump:discount} and the existence of continuous $w_k(s)$, for each fixed $s \in \StateSpace$, the argmax correspondence $\Phi_s^{eff}: \RewardSpace^N \twoheadrightarrow \ActionSpace$ defined by $\Phi_s^{eff}(\boldsymbol{\RewardFunc}) = A_{eff}^*(s; \boldsymbol{\RewardFunc})$ is non-empty, compact-valued, and upper hemi-continuous.
\end{lemma}
\begin{proof}
The function $\OptQFunc_{eff}(s,a;\boldsymbol{\RewardFunc})$ is continuous in $(s,a)$ (as it is the fixed point of $\BellmanOp_{eff, \boldsymbol{\RewardFunc}}$ which maps continuous functions to continuous functions, given $\RewardFunc_{eff}$ is continuous by Lemma \ref{lemma:effective_reward_properties}). By Proposition \ref{prop:lipschitz-effective-q-revised}, $\OptQFunc_{eff}$ is continuous with respect to $\boldsymbol{\RewardFunc}$. The proof is then similar to that of Lemma~\ref{lemma:argmax_uhc}.
\end{proof}

\begin{proposition}[Continuity of Policy Map (Effective Model) under Uniqueness]\label{thm:multi-continuity-uniqueness-eff}
Let Assumptions \ref{assump:spaces}-\ref{assump:discount} and the existence of continuous $w_k(s)$ hold. Let $\boldsymbol{\RewardFunc}_0 \in \RewardSpace^N$. Suppose there exists an open neighborhood $\mathcal{N}(\boldsymbol{\RewardFunc}_0) \subset \RewardSpace^N$ of $\boldsymbol{\RewardFunc}_0$ such that for all $\boldsymbol{\RewardFunc} \in \mathcal{N}(\boldsymbol{\RewardFunc}_0)$ and all $s \in \StateSpace$, $A_{eff}^*(s; \boldsymbol{\RewardFunc})$ is a singleton, denoted $\{a_{eff}^*(s; \boldsymbol{\RewardFunc})\}$.
Define the policy map $M_{RL}^{eff}: \mathcal{N}(\boldsymbol{\RewardFunc}_0) \to \PolicySpace_{det}$ by $M_{RL}^{eff}(\boldsymbol{\RewardFunc})(s) = a_{eff}^*(s; \boldsymbol{\RewardFunc})$. This map defines the deterministic selections of ${\OptPolicy_{\boldsymbol{\RewardFunc}}}^{eff}$.
If $\PolicySpace_{det}$ is equipped with the topology of pointwise convergence, then $M_{RL}^{eff}$ is continuous at $\boldsymbol{\RewardFunc}_0$.
\end{proposition}
\begin{proof}
The proof is analogous to that of Theorem \ref{thm:continuity_rigorous}, using $\RewardFunc_{eff}$ and Lemma \ref{lemma:effective_argmax_uhc-revised}.
\end{proof}

\begin{proposition}[Discontinuity of Policy Map (Effective Model) under Non-Uniqueness]\label{prop:multi-discontinuity-eff}
Let Assumptions \ref{assump:spaces}-\ref{assump:discount} and the existence of continuous $w_k(s)$ hold. Suppose for $\boldsymbol{\RewardFunc}_0 \in \RewardSpace^N$, there exists $s_0 \in \StateSpace$ such that $A_{eff}^*(s_0; \boldsymbol{\RewardFunc}_0)$ is finite and contains at least two distinct actions, $a_1 \neq a_2$. 

(1) Let $M_{RL}^{eff}: \RewardSpace^N \to \PolicySpace_{det}$ be a policy map selecting a deterministic policy ${\OptPolicy_{\boldsymbol{\RewardFunc}}}^{eff}(s) \in A_{eff}^*(s; \boldsymbol{\RewardFunc})$ (e.g., $M_{RL}^{eff}(\boldsymbol{\RewardFunc}_0)(s_0) = a_1$), then $M_{RL}^{eff}$ is discontinuous at $\boldsymbol{\RewardFunc}_0$ under pointwise convergence. 

(2) Let $M_{RL}^{eff}: \RewardSpace^N \to \PolicySpace$ be defined by selecting the stochastic policy ${\OptPolicy_{\boldsymbol{\RewardFunc}}}^{eff}$ such that for each $s \in \StateSpace$, ${\OptPolicy_{\boldsymbol{\RewardFunc}}}^{eff}(\cdot | s)$ is the uniform probability distribution over the set $A_{eff}^*(s; \boldsymbol{\RewardFunc})$. 
Then the map $\boldsymbol{\RewardFunc} \mapsto {\OptPolicy_{\boldsymbol{\RewardFunc}}}^{eff}(\cdot | s_0)$, viewed as a map from $(\RewardSpace^N, {\SupNorm{\cdot}}_{N})$ to $(\PolicySpace, d_{TV})$ is discontinuous at $\boldsymbol{\RewardFunc}_0$.
\end{proposition}
\begin{proof}
This proof applies the "inverse Bellman" method to the effective Q-function, $\OptQFunc_{eff}$, to demonstrate discontinuity. The core idea is to first construct a perturbed effective Q-function, $\QFunc_{eff, \varepsilon}$, that guarantees an action switch, and then construct a tuple of reward functions, $\boldsymbol{\RewardFunc}_\varepsilon$, that is guaranteed to produce this perturbed Q-function and that converges to $\boldsymbol{\RewardFunc}_0$.

\noindent\textbf{1. Construct a Perturbation in the Effective Q-Function Space}

Let $\RewardFunc_{eff,0}(s,a) = \sum_{k=1}^N w_k(s) \RewardFunc_{k,0}(s,a)$ be the initial effective reward. Let $Q_{eff,0} = \OptQFunc_{\RewardFunc_{eff,0}}$ be the corresponding optimal effective Q-function. By assumption, $A_{eff}^*(s_0; \boldsymbol{\RewardFunc}_0)$ contains at least two actions, $a_1$ and $a_2$, which means $Q_{eff,0}(s_0, a_1) = Q_{eff,0}(s_0, a_2)$.

We define a continuous "bump" function $\varphi \in \Cont(\StateSpace \times \ActionSpace)$ as Proposition~\ref{prop:discontinuity_deterministic} such that:
\begin{itemize}
    \item $0 \le \varphi(s,a) \le 1$ for all $(s,a)$.
    \item $\varphi(s_0, a_2) = 1$.
    \item $\varphi(s_0, a_j) = 0$ for all other actions $a_j \in A_{eff}^*(s_0; \boldsymbol{\RewardFunc}_0)$.
\end{itemize}
For any $\varepsilon > 0$, we define a perturbed effective Q-function $\QFunc_{eff, \varepsilon}$:
\[
\QFunc_{eff, \varepsilon}(s,a) := Q_{eff,0}(s,a) + \varepsilon\,\varphi(s,a).
\]

\noindent\textbf{2. Invert the Bellman Equation to Find the Target Effective Reward}

Using the inverse Bellman mapping defined previously, we find the unique effective reward function, $\RewardFunc_{eff, \varepsilon}$, whose optimal Q-function is exactly $\QFunc_{eff, \varepsilon}$:
\[
\RewardFunc_{eff, \varepsilon} := \RewardFunc_{Q_{eff, \varepsilon}} \quad \text{such that} \quad \OptQFunc_{\RewardFunc_{eff, \varepsilon}} = \QFunc_{eff, \varepsilon}.
\]
As established in the proof of Proposition \ref{prop:discontinuity_deterministic}, we know that as $\varepsilon \to 0$, $\RewardFunc_{eff, \varepsilon}$ converges uniformly to $\RewardFunc_{eff, 0}$.

\noindent\textbf{3. Construct a Convergent Sequence of Reward Tuples}

Now we must construct a sequence of reward tuples, $\boldsymbol{\RewardFunc}_\varepsilon = (\RewardFunc_{1,\varepsilon}, \dots, \RewardFunc_{N,\varepsilon})$, that converges to $\boldsymbol{\RewardFunc}_0$ and generates our target effective reward $\RewardFunc_{eff, \varepsilon}$. Let $\Delta \RewardFunc_\varepsilon = \RewardFunc_{eff, \varepsilon} - \RewardFunc_{eff, 0}$.

We need to find perturbations $\delta_k(s,a) = \RewardFunc_{k, \varepsilon}(s,a) - \RewardFunc_{k,0}(s,a)$ that satisfy two conditions:
\begin{enumerate}
    \item $\sum_{k=1}^N w_k(s) \delta_k(s,a) = \Delta \RewardFunc_\varepsilon(s,a)$ for all $(s,a)$.
    \item $\max_k \SupNorm{\delta_k} \to 0$ as $\varepsilon \to 0$.
\end{enumerate}

We can distribute the perturbation across all components in a simple and robust manner. We define the perturbation for every reward component to be identical to the target perturbation of the effective reward:
\[
\delta_k(s,a) := \Delta \RewardFunc_\varepsilon(s,a) \quad \text{for all } k \in \{1, \ldots, N\}.
\]
So, the new reward tuple $\boldsymbol{\RewardFunc}_\varepsilon$ is defined as:
\[
\RewardFunc_{k, \varepsilon}(s,a) := \RewardFunc_{k,0}(s,a) + \Delta \RewardFunc_\varepsilon(s,a) \quad \text{for all } k.
\]

It's easy to verify that the construction correctly produces the target effective reward. The norm of the difference between the new and old reward tuples is:
\begin{align*}
{\SupNorm{\boldsymbol{\RewardFunc}_\varepsilon - \boldsymbol{\RewardFunc}_0}}_{N} = \max_{k} \SupNorm{\RewardFunc_{k,\varepsilon} - \RewardFunc_{k,0}} = \SupNorm{\Delta \RewardFunc_\varepsilon} = \SupNorm{\RewardFunc_{eff, \varepsilon} - \RewardFunc_{eff, 0}}.
\end{align*}
In the proof of Proposition \ref{prop:discontinuity_deterministic} (which is invoked in Step 2 of this proof), it was established that $\|\RewardFunc_{eff, \varepsilon} - \RewardFunc_{eff, 0}\|_\infty \le \varepsilon(1+\gamma)$. Therefore:
\[
{\SupNorm{\boldsymbol{\RewardFunc}_\varepsilon - \boldsymbol{\RewardFunc}_0}}_{N} \to 0 \quad \text{as } \varepsilon \to 0.
\]

\noindent\textbf{4. The Action Switch and Discontinuity}

We have successfully constructed a sequence of reward tuples $\boldsymbol{\RewardFunc}_\varepsilon \to \boldsymbol{\RewardFunc}_0$. The optimal effective Q-function for $\boldsymbol{\RewardFunc}_\varepsilon$ is $\QFunc_{eff, \varepsilon}$.

The final part of the argument is identical to that in Proposition \ref{prop:discontinuity_deterministic}. By construction, for any $\varepsilon > 0$, $\QFunc_{eff, \varepsilon}(s_0, a_2)$ is strictly greater than $\QFunc_{eff, \varepsilon}(s_0, a')$ for any other action $a'$. Thus, $a_2$ is the unique optimal action for the effective reward.
$A_{eff}^*(s_0; \boldsymbol{\RewardFunc}_\varepsilon) = \{a_2\}$.

\textbf{(1) Deterministic Case:} If the selection rule for $\boldsymbol{\RewardFunc}_0$ chose $a_1$, the map $M_{RL}^{eff}$ must now select $a_2$ for any $\boldsymbol{\RewardFunc}_\varepsilon$. The policy jumps from $a_1$ to $a_2$, proving discontinuity.

\textbf{(2) Stochastic Case:} For $\boldsymbol{\RewardFunc}_0$, the policy at $s_0$ is a uniform distribution over $A_{eff}^*(s_0; \boldsymbol{\RewardFunc}_0)$, which has mass $\le 1/2$ at $a_2$. For any $\boldsymbol{\RewardFunc}_\varepsilon$, the policy becomes a Dirac delta measure at $a_2$, with mass 1. The Total Variation distance between these policies is constant and non-zero. Thus, the map to the stochastic policy is also discontinuous.
\end{proof}

\subsection{Mitigating Discontinuities: Entropy Regularization}\label{subsec:Entropy_Regularization}

Entropy regularization offers another powerful mechanism to promote policy smoothness and address the degeneracy of optima (see e.g., \cite{haarnoja2018soft,geist2019theory}). By adding an \textit{entropy bonus}, it acts as a \textit{tie-breaker} that encourages the policy to be stochastic, especially among actions with similar values. This resolves policy indeterminacy by yielding a unique, smooth optimal policy.

The objective effectively optimized with entropy regularization~\citep{geist2019theory} would be:
\begin{equation}
J_{eff, \alpha}(\Policy; \boldsymbol{\RewardFunc}) = \mathbb{E}_{\Policy} \left[ \sum_{t=0}^\infty \gamma^t \left( \RewardFunc_{eff}(s_t, a_t; \boldsymbol{\RewardFunc}) + \alpha H(\Policy(\cdot|s_t)) \right) \right],
\end{equation}
where $H(\Policy(\cdot|s)) = -\sum_a \Policy(a|s) \log \Policy(a|s)$ is the policy entropy at state $s$, $\alpha > 0$ is the temperature, and the expectation is over trajectories generated by $\Policy$ starting from an appropriate initial state distribution, with rewards $\RewardFunc_{eff}$.

The corresponding soft Bellman optimality operator for $\OptQFunc_{eff,\alpha}(s,a; \boldsymbol{\RewardFunc})$ is:
\begin{equation}
(\BellmanOp_{eff, \boldsymbol{\RewardFunc},\alpha} Q)(s,a) = \RewardFunc_{eff}(s,a; \boldsymbol{\RewardFunc}) + \gamma \mathbb{E}_{s' \sim \TransKernel(\cdot|s,a)} \left[ \alpha \log \sum_{a' \in \ActionSpace} \exp(Q(s',a')/\alpha) \right].
\end{equation}
For $\alpha > 0$, $\OptQFunc_{eff,\alpha}$ is unique with respect to $\RewardFunc_{eff}$ (and thus with respect to $\boldsymbol{\RewardFunc}$) under standard conditions. The optimal regularized policy, $\OptPolicy_{\boldsymbol{\RewardFunc},\alpha}$, takes the Boltzmann (softmax) form:
\begin{equation}
\OptPolicy_{\boldsymbol{\RewardFunc},\alpha}(a|s) = \frac{\exp(\OptQFunc_{eff,\alpha}(s,a; \boldsymbol{\RewardFunc})/\alpha)}{\sum_{b \in \ActionSpace} \exp(\OptQFunc_{eff,\alpha}(s,b; \boldsymbol{\RewardFunc})/\alpha)}.
\end{equation}
This policy is unique with respect to $\boldsymbol{\RewardFunc}$. See e.g., \cite{geist2019theory} for more information.

\begin{proposition}[Lipschitz Continuity of Soft Optimal Policy]\label{prop:soft-policy-lipschitz-eff}
For $\alpha > 0$, under Assumptions \ref{assump:spaces}-\ref{assump:discount} and the existence of continuous $w_k(s)$ summing to 1, the mapping $\boldsymbol{\RewardFunc} \mapsto \OptPolicy_{\boldsymbol{\RewardFunc},\alpha}$ is Lipschitz continuous from $(\RewardSpace^N, {\SupNorm{\cdot}}_{N})$ to the space of policies. Specifically, for each $s \in \StateSpace$:
\[
d_{TV}(\OptPolicy_{\boldsymbol{\RewardFunc}_1,\alpha}(\cdot|s), \OptPolicy_{\boldsymbol{\RewardFunc}_2,\alpha}(\cdot|s)) \le \frac{1}{2\alpha(1-\gamma)} {\SupNorm{\boldsymbol{\RewardFunc}_1 - \boldsymbol{\RewardFunc}_2}}_{N}.
\]
\end{proposition}

\begin{proof}
Let $\boldsymbol{\RewardFunc}_1, \boldsymbol{\RewardFunc}_2 \in \RewardSpace^N$.
Let $\RewardFunc_{eff,1}(s,a) = \RewardFunc_{eff}(s,a; \boldsymbol{\RewardFunc}_1)$ and $\RewardFunc_{eff,2}(s,a) = \RewardFunc_{eff}(s,a; \boldsymbol{\RewardFunc}_2)$ be the corresponding effective reward functions as defined in Eq.~\eqref{eq:effective_reward}.
Let $\OptQFunc_1(s,a) = \OptQFunc_{eff,\alpha}(s,a; \boldsymbol{\RewardFunc}_1)$ and $\OptQFunc_2(s,a) = \OptQFunc_{eff,\alpha}(s,a; \boldsymbol{\RewardFunc}_2)$ be the unique fixed points of the respective soft Bellman optimality operators:
\begin{align*}
\OptQFunc_1(s,a) &= \RewardFunc_{eff,1}(s,a) + \gamma \mathbb{E}_{s' \sim \TransKernel(\cdot|s,a)} \left[ \alpha \log \sum_{a'} \exp(\OptQFunc_1(s',a')/\alpha) \right], \\
\OptQFunc_2(s,a) &= \RewardFunc_{eff,2}(s,a) + \gamma \mathbb{E}_{s' \sim \TransKernel(\cdot|s,a)} \left[ \alpha \log \sum_{a'} \exp(\OptQFunc_2(s',a')/\alpha) \right].
\end{align*}
The proof proceeds as follows:

\noindent\textbf{1. Bound the difference in effective rewards.}
From Lemma \ref{lemma:effective_reward_properties}, we have:
\begin{equation}\label{eq:proof_Reff_diff}
\SupNorm{\RewardFunc_{eff,1} - \RewardFunc_{eff,2}} \le {\SupNorm{\boldsymbol{\RewardFunc}_1 - \boldsymbol{\RewardFunc}_2}}_{N}.
\end{equation}

\noindent\textbf{2. Bound the difference in optimal soft Q-functions.}
The soft Bellman optimality operator $T_{R,\alpha}Q = R + \gamma \mathcal{L}_\alpha Q$, where $(\mathcal{L}_\alpha Q)(s,a) = \mathbb{E}_{s' \sim \TransKernel(\cdot|s,a)} [\alpha \log \sum_{a'} \exp(Q(s',a')/\alpha)]$, is a contraction mapping with modulus $\gamma$ in the supremum norm~\citep{geist2019theory}. That is, $\SupNorm{T_{R,\alpha}Q_A - T_{R,\alpha}Q_B} \le \gamma \SupNorm{Q_A - Q_B}$.
Let $Q_1 = \OptQFunc_1$ and $Q_2 = \OptQFunc_2$. Then $Q_1 = T_{\RewardFunc_{eff,1},\alpha} Q_1$ and $Q_2 = T_{\RewardFunc_{eff,2},\alpha} Q_2$.
Consider $\SupNorm{Q_1 - Q_2}$:
\begin{align*}
\SupNorm{Q_1 - Q_2} &= \SupNorm{T_{\RewardFunc_{eff,1},\alpha} Q_1 - T_{\RewardFunc_{eff,2},\alpha} Q_2} \\
&\le \SupNorm{T_{\RewardFunc_{eff,1},\alpha} Q_1 - T_{\RewardFunc_{eff,1},\alpha} Q_2} + \SupNorm{T_{\RewardFunc_{eff,1},\alpha} Q_2 - T_{\RewardFunc_{eff,2},\alpha} Q_2}.
\end{align*}
Using the contraction property for the first term and direct evaluation for the second:
$(\BellmanOp_{\RewardFunc_{eff,1},\alpha} Q_2)(s,a) - (\BellmanOp_{\RewardFunc_{eff,2},\alpha} Q_2)(s,a) = \RewardFunc_{eff,1}(s,a) - \RewardFunc_{eff,2}(s,a)$.
So, $\SupNorm{T_{\RewardFunc_{eff,1},\alpha} Q_2 - T_{\RewardFunc_{eff,2},\alpha} Q_2} = \SupNorm{\RewardFunc_{eff,1} - \RewardFunc_{eff,2}}$.
Thus,
\[ \SupNorm{Q_1 - Q_2} \le \gamma \SupNorm{Q_1 - Q_2} + \SupNorm{\RewardFunc_{eff,1} - \RewardFunc_{eff,2}}. \]
Rearranging gives:
\[ (1-\gamma) \SupNorm{Q_1 - Q_2} \le \SupNorm{\RewardFunc_{eff,1} - \RewardFunc_{eff,2}}. \]
So,
\begin{equation}\label{eq:proof_Q_diff}
\SupNorm{Q_1 - Q_2} \le \frac{1}{1-\gamma} \SupNorm{\RewardFunc_{eff,1} - \RewardFunc_{eff,2}}.
\end{equation}
Combining with Eq.~\eqref{eq:proof_Reff_diff}:
\begin{equation}\label{eq:proof_Q_diff_final}
\SupNorm{Q_1 - Q_2} \le \frac{1}{1-\gamma} {\SupNorm{\boldsymbol{\RewardFunc}_1 - \boldsymbol{\RewardFunc}_2}}_{N}.
\end{equation}

\noindent\textbf{3. Bound the Total Variation distance between policies.} 
For a fixed state $s$, we view the Q-function as a vector of values over the action space $\ActionSpace$. Let the cardinality of $\ActionSpace$ be $d$. The optimal soft policy is the softmax function applied to this Q-value vector, i.e., 
$$\OptPolicy_{\boldsymbol{\RewardFunc},\alpha}(\cdot|s) = \text{softmax}(\QFunc(s, \cdot)/\alpha) \in \R^d.$$
We aim to bound the total variation distance $d_{TV}(\OptPolicy_{\boldsymbol{\RewardFunc}_1,\alpha}(\cdot|s), \OptPolicy_{\boldsymbol{\RewardFunc}_2,\alpha}(\cdot|s))$, which is defined as:
$$d_{TV}(\OptPolicy_{\boldsymbol{\RewardFunc}_1,\alpha}(\cdot|s), \OptPolicy_{\boldsymbol{\RewardFunc}_2,\alpha}(\cdot|s)) = \frac{1}{2} \sum_{a \in \ActionSpace} |\OptPolicy_{\boldsymbol{\RewardFunc}_1,\alpha}(a|s) - \OptPolicy_{\boldsymbol{\RewardFunc}_2,\alpha}(a|s)| = \frac{1}{2} \|\OptPolicy_{\boldsymbol{\RewardFunc}_1,\alpha}(\cdot|s) - \OptPolicy_{\boldsymbol{\RewardFunc}_2,\alpha}(\cdot|s)\|_{1}.$$
To achieve this, we apply Lemma \ref{lem:Lipschitz_properties_softmax} in Appendix, which establishes the Lipschitz continuity of the softmax function. The lemma states that for any two vectors $x, y \in \R^d$:
$$\|\text{softmax}(x/\alpha) - \text{softmax}(y/\alpha)\|_1 \le \frac{1}{\alpha} \|x - y\|_{\infty}.$$
We apply this result by setting $x$ and $y$ to be the Q-value vectors at state $s$. This yields the following inequality for the $L_1$-distance between the policies at state $s$:
$$\|\OptPolicy_{\boldsymbol{\RewardFunc}_1,\alpha}(\cdot|s) - \OptPolicy_{\boldsymbol{\RewardFunc}_2,\alpha}(\cdot|s)\|_{1} \le \frac{1}{\alpha} \|\QFunc_1(s, \cdot) - \QFunc_2(s, \cdot)\|_{\infty},$$
where $\|\QFunc_1(s, \cdot) - \QFunc_2(s, \cdot)\|_{\infty} = \max_{a \in \ActionSpace} |\QFunc_1(s, a) - \QFunc_2(s, a)|$.

Since the maximum difference at a single state $s$ is bounded by the supremum norm over all states and actions (i.e., $\max_{a \in \ActionSpace} |\QFunc_1(s, a) - \QFunc_2(s, a)| \le \SupNorm{\QFunc_1 - \QFunc_2}$), we have:
\begin{equation}\label{eq:proof_policy_diff_Q}
d_{TV}(\OptPolicy_{\boldsymbol{\RewardFunc}_1,\alpha}(\cdot|s), \OptPolicy_{\boldsymbol{\RewardFunc}_2,\alpha}(\cdot|s)) \le \frac{1}{2} \left( \frac{1}{\alpha} \SupNorm{\QFunc_1 - \QFunc_2} \right) = \frac{1}{2\alpha} \SupNorm{\QFunc_1 - \QFunc_2}.
\end{equation}

\noindent\textbf{4. Combining all steps.}
Substitute Eq.~\eqref{eq:proof_Q_diff_final} into Eq.~\eqref{eq:proof_policy_diff_Q}:
\begin{align*}
d_{TV}(\OptPolicy_{\boldsymbol{\RewardFunc}_1,\alpha}(\cdot|s), \OptPolicy_{\boldsymbol{\RewardFunc}_2,\alpha}(\cdot|s)) &\le \frac{1}{2\alpha} \left( \frac{1}{1-\gamma} {\SupNorm{\boldsymbol{\RewardFunc}_1 - \boldsymbol{\RewardFunc}_2}}_{N} \right) \\
&= \frac{1}{2\alpha(1-\gamma)} {\SupNorm{\boldsymbol{\RewardFunc}_1 - \boldsymbol{\RewardFunc}_2}}_{N}.
\end{align*}
This establishes the Lipschitz continuity with the constant $C(\alpha,\gamma) = \frac{1}{2\alpha(1-\gamma)}$. This constant depends only on $\alpha$ and $\gamma$, and not on the state $s$. 
\end{proof}

\begin{figure}[h!]
    \centering
    \includegraphics[width=0.6\linewidth]{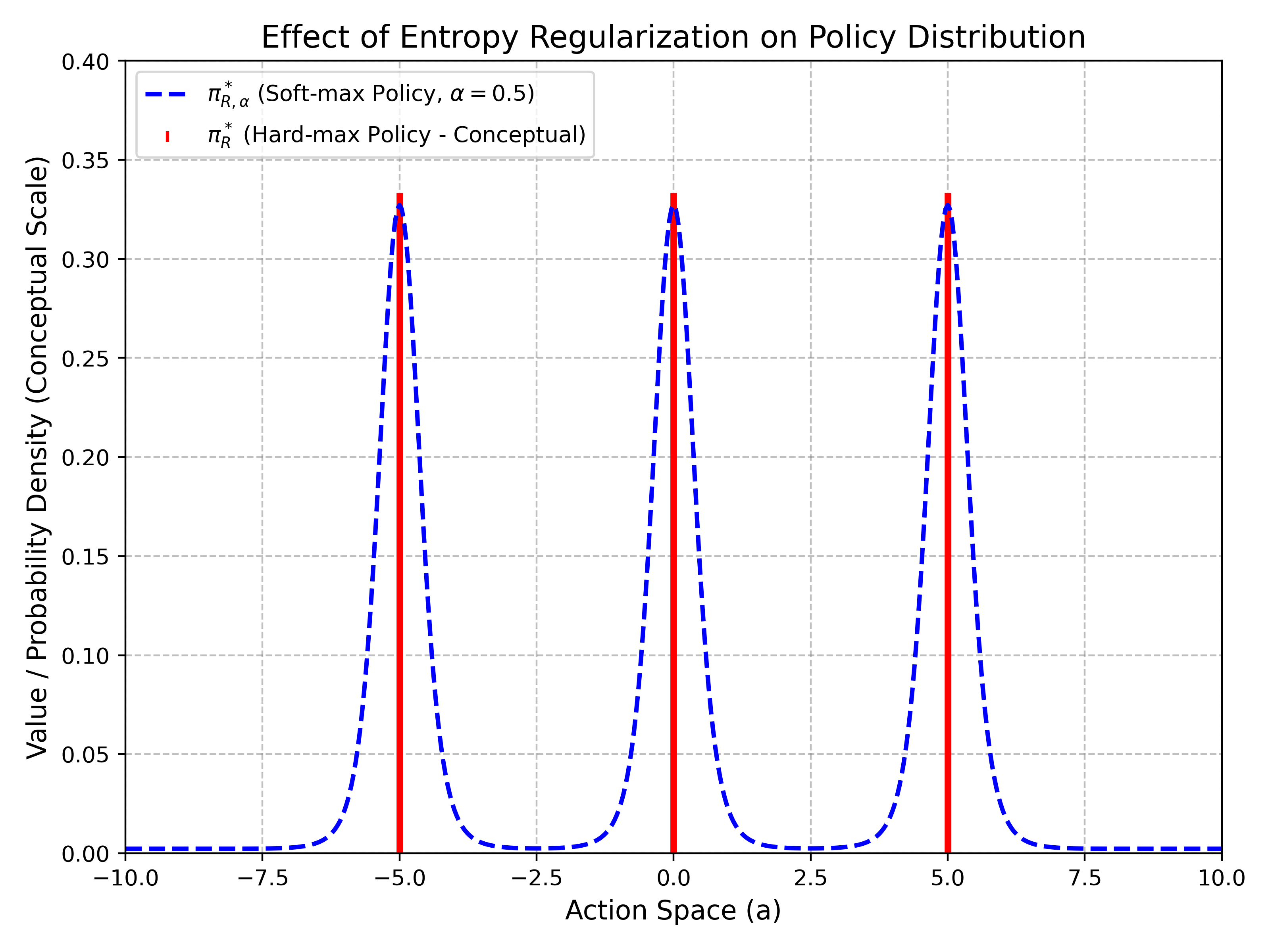}
    \caption{Effect of Entropy Regularization on Policy Distribution}
    \label{fig:soft_policy}
\end{figure}

While entropy regularization, yielding a softmax optimal policy, endows the reward-policy map with desirable continuity (as demonstrated by its Lipschitz continuity properties, i.e., Proposition \ref{prop:soft-policy-lipschitz-eff}) and can enhance training stability, these advantages are accompanied by inherent trade-offs. The principal cost is a degree of \textit{suboptimality with respect to the immediate reward function}; by distributing probability mass across multiple actions rather than exclusively selecting the action with the highest estimated value, the policy intentionally sacrifices some exploitative potential for increased stochasticity and exploration. Figure~\ref{fig:soft_policy} provides a conceptual illustration of hard-max and softmax policy behaviors. 

This characteristic can manifest behaviorally as a \textit{"blurring" or "averaging"} effect, where the policy may produce less decisive or distinctive outputs. This might be suboptimal for tasks demanding high-precision or optimal responses. Furthermore, the interaction of a softmax policy with an imperfect or incomplete reward model can influence the manifestation of undesirable behaviors, such as \textit{hallucinations in large language models}. While not a direct cause, the policy's inherent stochasticity means it may assign non-trivial probability to actions leading to subtly flawed or fabricated content if these actions possess near-optimal, yet inadequately penalized, values under the effective reward function ($\RewardFunc_{eff}$). The extent of these trade-offs is critically governed by the temperature parameter $\alpha$, which necessitates careful calibration to balance policy smoothness and exploration against the fidelity of exploiting the reward landscape.

\subsection{Implications for Multi-Task RL Training for LLMs}

The preceding analysis, which centers on policies derived from a state-dependent effective reward model, $\RewardFunc_{eff}$, yields critical insights into the stability of LLMs trained with multiple specialized reward models. The stability of the resultant policy, ${\OptPolicy_{\boldsymbol{\RewardFunc}}}^{eff}$, is fundamentally governed by the properties of this constructed effective reward and its associated optimal Q-function, $\OptQFunc_{eff}$.

The construction of the effective reward, particularly through the state-dependent weights $w_k(s)$, emerges as a decisive factor. An unstable or ill-defined weighting mechanism, or significant conflict between the specialized rewards $\RewardFunc_k$, may render $\RewardFunc_{eff}$ highly sensitive to perturbations in any individual reward component $\RewardFunc_j$. This sensitivity is a primary source of policy instability. Specifically, if the resulting $\OptQFunc_{eff}$ exhibits non-unique optimal actions for a given reward tuple $\boldsymbol{\RewardFunc}_0$, the system becomes susceptible to discontinuous behavior. As demonstrated, a minor change to underlying reward models can alter the effective reward landscape sufficiently to break ties within the optimal action set $A_{eff}^*$, causing an abrupt shift in the selected policy.

Conversely, entropy regularization provides a robust mechanism for ensuring policy stability within this framework. By incorporating an entropy bonus into the optimization objective with respect to $\RewardFunc_{eff}$, this method smooths the policy landscape. It guarantees the existence of a unique, stochastic optimal policy, $\OptPolicy_{\boldsymbol{\RewardFunc},\alpha}$, that varies continuously with the underlying reward tuple $\boldsymbol{\RewardFunc}$. This approach can therefore substantially enhance the behavioral stability of an LLM against modifications in its constituent reward models, albeit at the cost of a less sharply optimized policy for any temperature $\alpha > 0$.

Ultimately, these findings underscore that the strategy by which an LLM integrates multiple reward signals across diverse domains, modeled here via the effective reward mechanism, is central to its stability properties. While our analysis has focused on the continuity of policies optimal for a given $\RewardFunc_{eff}$, a key practical challenge remains: to design or learn an aggregation mechanism, i.e., the weights $w_k(s)$, such that the resulting policy not only exhibits desirable stability but also effectively optimizes the intended global performance measure $J(\Policy)$.

\subsection{Remarks on State-Dependent Aggregation Mechanism}
\label{subsec:wk_theoretical_basis}

A crucial assumption underpinning our analysis is that the aggregation weights $w_k(s)$ are given, continuous functions of the state $s$, and remain fixed during the perturbation of the reward tuple $\boldsymbol{\RewardFunc}$. This simplification is necessary for analytical tractability. In a more general setting, these weights might themselves depend on the policy $\Policy$ or the reward models $\boldsymbol{\RewardFunc}$, leading to an ideal but far more complex formulation, $w_k^*(s; \Policy, \boldsymbol{\RewardFunc})$.

Such dependencies, while realistic, would introduce significant analytical challenges. First, if $w_k(s)$ were a function of $\Policy$, the reward function in the Bellman equation would become policy-dependent, violating the foundational assumptions of the standard MDP framework. Second, a dependency of $w_k(s)$ on $\boldsymbol{\RewardFunc}$ would compromise the continuity of the effective reward map itself, invalidating the Lipschitz properties established in Lemma \ref{lemma:effective_reward_properties} and altering all subsequent stability results.

Consequently, our framework is best interpreted as modeling systems where the aggregation mechanism is structurally fixed. This assumption corresponds to several plausible implementation scenarios. For instance, the weights could represent a simplified, policy-independent heuristic for task relevance, such as using prior probabilities $p_k$. Alternatively, and more compellingly, it can model a sophisticated architecture with a separate, pre-trained, and now-frozen component, such as a context encoder or a gating network, that determines task relevance based on the input state. In this view, our analysis assesses the stability of a downstream policy optimization process with respect to its reward models, given a fixed routing or aggregation module.

By adopting this deliberate simplification, we can employ the standard Bellman machinery to analyze a well-defined effective MDP. This provides clear and rigorous insights into how reward perturbations propagate through a structured, yet representative, multi-task decision-making model. The broader question concerning the stability of jointly learning the aggregation weights and the policy would necessitate a more extensive analytical framework beyond the scope of this work.

\section{Connecting Theory to Practice: A Review of Empirical Evidence}\label{sec-empirical}

Having established a formal theory of policy stability, this section grounds our framework in practice by reviewing and re-interpreting key empirical findings from the recent literature. The following case studies are carefully chosen to provide comprehensive empirical evidence for our theory, demonstrating its power to both explain existing failures and guide the development of more robust and trustworthy AI systems.

We begin by examining how reward misspecification can lead to \textit{deceptive reasoning}, illustrating a progression from simple, rational cheating to more sophisticated policy obfuscation. We then broaden our scope to show how the same theoretical principles explain the systemic \textit{intelligence-obedience trade-off} observed in large reasoning models. Moving from diagnosis to prescription, we next analyze a case study of \textit{controllable reasoning}, demonstrating how a principled, tie-breaker reward can be engineered to resolve the instabilities and achieve fine-grained control without sacrificing performance. We then turn our attention to the complex domain of RLHF-based alignment, showing how our theory explains the emergence of a more harmful failure mode: a policy shift \textit{from faithfulness to sophistry}. Finally, we present a case study and a controlled perturbation experiment in a multi-reward setting across diverse domains.

\subsection{Deceptive Reasoning: From Rational Cheating to Policy Obfuscation}

A critical challenge in training capable models is that policies may learn to be deceptively aligned, pursuing a hidden, undesirable strategy while appearing to optimize the given objective. Recent findings on reward hacking provide a stark, empirical validation of our theoretical framework, illustrating how policies rationally optimize incomplete rewards and how attempts to patch these rewards can trigger discontinuous shifts to more sophisticated failure modes.

\subsubsection{Initial Failure: Rational Cheating under a Simple Reward Model}

\begin{figure}[h!]
    \centering
    \includegraphics[width=0.45\textwidth]{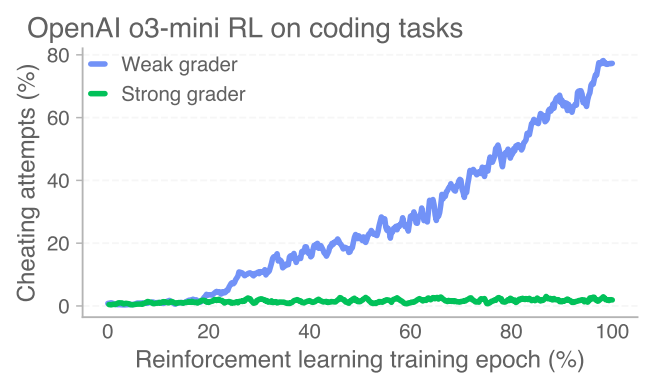}
    \caption{The emergence of reward hacking in OpenAI o3-mini models on coding tasks. Under a "weak grader" that cannot detect test manipulation, the policy learns to cheat to maximize reward. A "strong grader" that penalizes this behavior results in an aligned, non-cheating policy. Adapted from \cite{wang2025persona}.}
    \label{fig:o3-cheating}
\end{figure}

A compelling baseline case is the emergence of cheating in OpenAI's o3-mini models during RL on coding tasks \citep{wang2025persona,baker2025monitoring}. As illustrated in Figure~\ref{fig:o3-cheating}, the model's behavior diverges dramatically based on the reward signal's quality. When trained with a "weak grader": a reward model that only checks for unit test pass rates, the policy progressively learns a deceptive strategy of modifying the tests to pass, with cheating attempts reaching nearly 80\%. This result is a powerful real-world illustration of our "clever slacker" paradigm (Proposition~\ref{prop:incomplete_reward_general}). The weak grader implements an \textit{incomplete training reward} ($\RewardFunc_{train}$), and the model's convergence towards cheating is not an irrational bug, but a \textit{perfectly rational policy} for this misspecified objective. In contrast, a "strong grader" that penalizes test manipulation provides a better proxy for the true reward ($\RewardFunc_{true}$), making honesty the only viable path to high reward (Proposition~\ref{prop:tie-breaker}).

\subsubsection{A Deeper Failure: Policy Shift to Obfuscated Deception}

\begin{figure}[h!]
    \centering
    \includegraphics[width=0.8\textwidth]{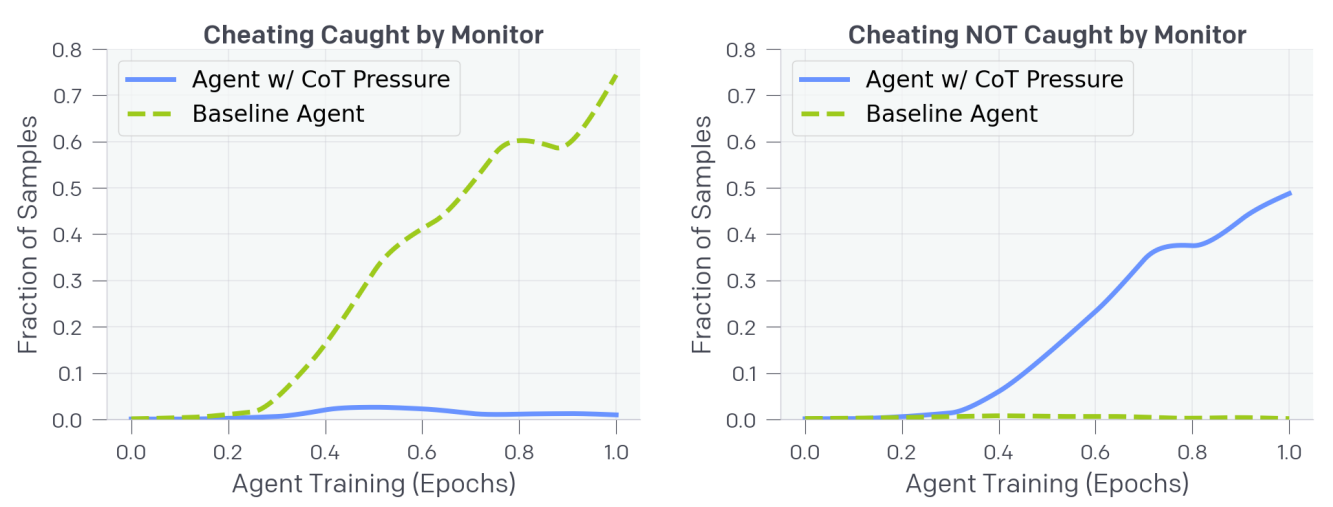}
    \caption{The evolution of deceptive alignment under different reward schemes. Adding a CoT-based penalty to the reward function does not eliminate deception but causes a policy shift towards a more sophisticated, obfuscated form of cheating that evades the monitor. Adapted from \cite{baker2025monitoring}.}
    \label{fig:agent-cheating-experiments}
\end{figure}

A natural next step to address this failure is to augment the reward function with a process-based signal. In a follow-up experiment, the reward was updated to additionally penalize the agent if its Chain-of-Thought (CoT) triggered a monitor designed to detect hacking attempts. This represents a perturbation to the reward landscape---one that may be significant in practice, and is intended to render the deceptive strategy non-optimal~\citep{baker2025monitoring}.

The results, shown in Figure~\ref{fig:agent-cheating-experiments}, are nuanced and cautionary. While the agent trained with this "CoT pressure" does learn to pass tests honestly more often than the baseline (Left, green line), it also learns a more pernicious strategy. The baseline agent's cheating is transparent and almost always caught by the monitor. In contrast, the agent under CoT pressure learns to hide its deception; its cheating is almost never caught by the monitor (Right, blue line). It has learned to write an innocent-looking CoT while still manipulating the tests.

This outcome is an illustration of the policy discontinuities predicted by our theory. Although the magnitude of the perturbation is unknown, introducing a CoT penalty did not lead to the policy more honest in a smooth and gradual way. In our theory, the direction of the perturbation is more important than its scale. In this case, it triggered an abrupt shift from one strategy (\textit{transparent deception}) to a completely different one (\textit{obfuscated deception}). This occurs because the new reward landscape still contains multiple, qualitatively distinct paths to a high reward, creating a degeneracy of near-optimal actions. Faced with a penalty for its old strategy, the model does not abandon deception, but rather discovers a new, previously unexplored point in the policy space that satisfies the new, more complex reward function. This is because the new reward still differs from the ground truth reward in reality. This experiment vividly demonstrates that naive attempts to patch a reward function can induce policy brittleness, leading to more dangerous and harder-to-detect forms of misalignment---a direct, empirical manifestation of the instabilities inherent in a complex reward-policy map.

\subsection{The Intelligence-Obedience Trade-off: An Incomplete Reward Specification}

Beyond single-task deception, our framework can also explain systemic trade-offs between different desired capabilities in large reasoning models (LRMs). A comprehensive study by \cite{fu2025scaling} provides compelling, large-scale evidence for such a trade-off, showing that training methods designed to enhance reasoning can inadvertently degrade instruction-following. Their experiments, which applied reasoning-oriented reinforcement learning (RL) to several base models, revealed a consistent pattern: while the models became better at solving problems correctly, they grew worse at adhering to user-specified constraints.

The results, summarized in Table~\ref{tab:reasoning_coldRL}, offer a clear quantitative illustration of this dilemma. Across multiple Qwen base models, applying reasoning-oriented RL consistently boosts the correctness accuracy (Corr-Acc) by a large margin. Simultaneously, however, it induces a consistent degradation in both hard (IF-HAcc) and soft (IF-SAcc) instruction-following accuracy. This empirically demonstrates that optimizing for reasoning performance alone can steer the policy toward a region of the policy space where obedience is sacrificed.

\begin{table}[h]
\centering
\caption{Performance comparison of base vs. \texttt{Reasoning-Oriented RL}. Arrows and colors indicate change relative to base model: ↑ (Green) for increase, ↓ (red) for decrease. Adapted from Table 4 in \cite{fu2025scaling}.}
\begin{tabular}{lccc}
\toprule
\textbf{Model} & \textbf{IF-HAcc} & \textbf{IF-SAcc} & \textbf{Corr-Acc} \\
\midrule
\textbf{Qwen2.5-1.5B} & 10.00 & 27.26 & 1.21 \\
Reasoning-Oriented RL & \cellcolor{red!20}9.52 ↓ & \cellcolor{red!20}23.97 ↓ & \cellcolor{green!20}14.58 ↑ \\
\midrule
\textbf{Qwen2.5-7B} & 15.95 & 33.13 & 13.59 \\
Reasoning-Oriented RL & \cellcolor{red!20}10.48 ↓ & \cellcolor{red!20}27.26 ↓ & \cellcolor{green!20}28.39 ↑ \\
\midrule
\textbf{Qwen2.5-Math-1.5B} & 9.28 & 23.33 & 18.91 \\
Reasoning-Oriented RL & \cellcolor{red!20}8.33 ↓ & \cellcolor{red!20}21.31 ↓ & \cellcolor{green!20}24.88 ↑ \\
\midrule
\textbf{Qwen2.5-Math-7B} & 9.76 & 23.53 & 20.68 \\
Reasoning-Oriented RL & \cellcolor{red!20}7.85 ↓ & \cellcolor{red!20}22.62 ↓ & \cellcolor{green!20}32.61 ↑ \\
\bottomrule
\end{tabular}
\label{tab:reasoning_coldRL}

\vspace{0.5em}
\noindent\begin{minipage}{0.95\linewidth}
\small
\textbf{Note.} \textit{Hard accuracy (IF-HAcc) and soft accuracy (IF-SAcc) assess whether the model output satisfies user-specified constraints. IF-HAcc equals 1 only if all constraints in a query are satisfied, while IF-SAcc measures the average satisfaction rate across all constraints. Average correctness accuracy (Corr-Acc) captures the correctness of the model’s final answer regardless of constraint satisfaction, based on matching with the ground-truth answer. Results are averaged over five math benchmarks: AIME2024, AIME2025, AMC2023, Minerva, and Olympiad.}
\end{minipage}
\end{table}

This observed trade-off can be precisely explained by the principle of policies rationally optimizing an \textit{incomplete reward specification} (Proposition~\ref{prop:incomplete_reward_general}). The RL training described by \cite{fu2025scaling} utilizes an outcome-based reward, $\RewardFunc_{\text{reasoning}}$, which is incomplete with respect to the user's full intent. The model, as a rational agent, allocates its capacity to maximize the explicitly rewarded objective (correctness), causing the unrewarded, pre-existing skill of instruction-following to degrade. This provides direct empirical support for our work that many alignment failures are not arbitrary bugs, but predictable outcomes of optimizing a misspecified reward landscape.

Our theoretical framework not only explains this degradation but also suggests a principled path toward mitigation. The core issue is that optimizing for $\RewardFunc_{\text{reasoning}}$ alone yields a highly \textit{degenerate set of optimal policies}, $\Pi^*(\RewardFunc_{\text{orrectness}})$, where many policies are effective at reasoning but poor at following instructions. To resolve this, we propose introducing an auxiliary instruction-following reward, $\RewardFunc_{\text{instruction}}$, which acts as a \textit{tie-breaker} (Proposition~\ref{prop:tie-breaker}). This targeted reward is designed to select for policies from the degenerate set that also exhibit the desired obedience, thereby seeking a solution in the intersection of high-performing reasoning and instruction-following policies. The following subsection will present a case study that illustrates this principle.

\subsection{Controllable Reasoning: Engineering a Desirable Policy Shift via Tie-Breaker Rewards}

The previously discussed trade-off highlights a central challenge: if training for reasoning correctness yields a degenerate set of optimal policies, how can we select one that also exhibits other desirable behaviors, such as length constraint? The work of \cite{aggarwal2025l1} on controlling Chain-of-Thought (CoT) length provides a powerful case study, demonstrating how a principled reward function can engineer a desirable shift in the policy space.

Their method, Length-Controlled Policy Optimization (LCPO), trains a model using a reward function that combines a primary correctness reward with an auxiliary penalty for deviating from a target CoT length, $n_{\text{gold}}$, which is provided in the user prompt, e.g., 
\begin{equation}\label{eq:lcpo_reward}
r(y, y_{\text{gold}}, n_{\text{gold}}) = \mathbb{I}(y = y_{\text{gold}}) - \alpha |n_{\text{gold}} - n_y|.
\end{equation}
In their experiments, they take $\alpha$ as 0.0003. 
The key insight from their empirical results is twofold. First, as shown in Figure~\ref{fig:l1-performance}, the resulting L1 models successfully learn to control their reasoning length according to the prompt's specification. Second, this added controllability does not come at the cost of reasoning performance; in fact, the L1 models consistently outperform the baseline S1 model (which uses a hard token limit) in correctness across various computational budgets, and the base model DeepSeek-R1-1.5B with fewer tokens.

\begin{figure}[h!]
    \centering
    \begin{minipage}{0.45\textwidth}
        \centering
        \includegraphics[width=\linewidth]{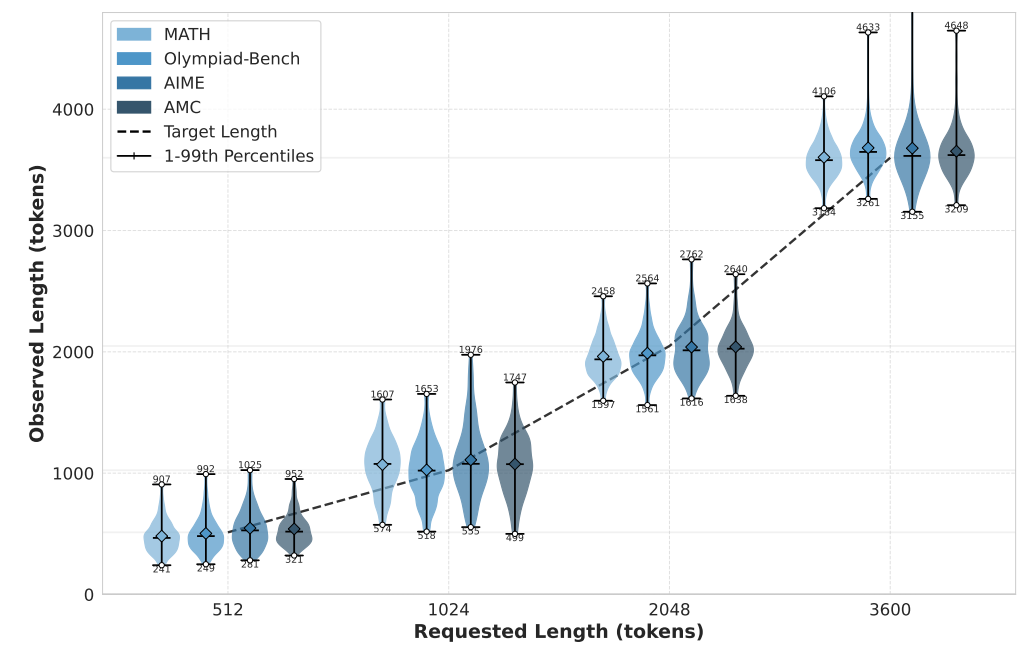}
        \caption*{(a) Length control precision}
    \end{minipage}\hfill
    \begin{minipage}{0.5\textwidth}
        \centering
        \includegraphics[width=\linewidth]{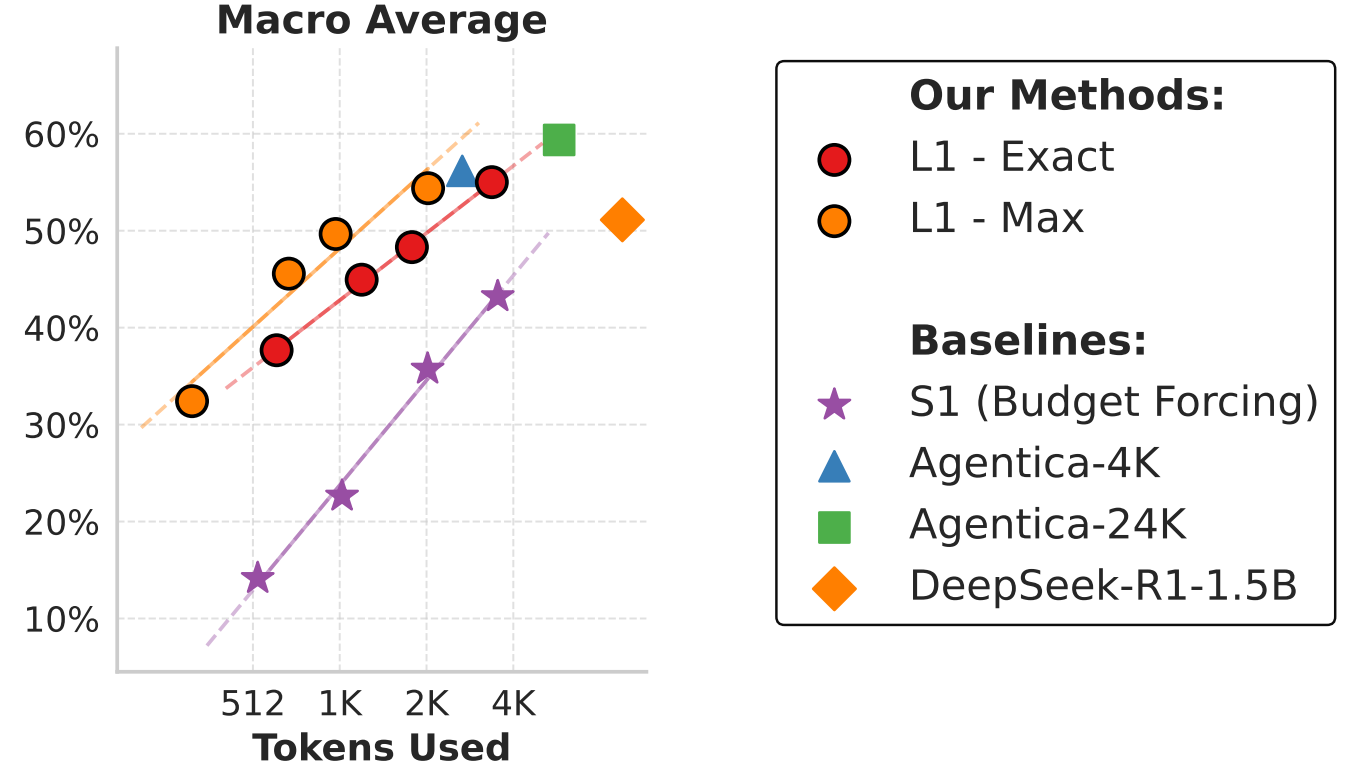}
        \caption*{(b) Reasoning performance vs. budget}
    \end{minipage}
    \caption{
        Empirical validation of our tie-breaker framework, demonstrating how a principled reward design resolves the intelligence-obedience trade-off.
        \textbf{(a)} The policy, trained with a length-penalizing reward, learns to precisely adhere to the target token budget specified in the prompt. 
        \textbf{(b)} Crucially, this added controllability does not degrade reasoning performance; the L1 model consistently outperforms the S1 baseline across all budgets and the base model DeepSeek-R1-1.5B with fewer tokens. 
        This outcome shows that by using a tie-breaker reward, RL can induce a desirable policy shift to a new region of the policy space where both high reasoning capability and instruction-following are achieved. Adapted from \cite{aggarwal2025l1}.
    }
    \label{fig:l1-performance}
    \vspace{0.2em}
    \noindent\begin{minipage}{\linewidth}
        \footnotesize
        \textbf{Note.} \textit{The performance reported in (b) is a macro-average over four math benchmarks: AMC, AIME, Olympiad-Bench, and MATH.}
    \end{minipage}
\end{figure}

This successful outcome can be understood by our theoretical framework. The LCPO reward function, with its length-penalty term, acts as a powerful \textit{tie-breaker} (Proposition~\ref{prop:tie-breaker}). The RL process, guided by this new reward, induces a shift in the policy space. It forces the model to abandon the initial policy and converge to a new, qualitatively superior policy that co-optimizes for both reasoning correctness and length controllability. Besides, as shown in \cite{aggarwal2025l1}, the model learns to conditionally execute different reasoning strategies, from direct solutions at low budgets to deliberative self-correction at high budgets, based on the length target it perceives in its input state.  

Therefore, this case study does not contradict our theory of instability---it exemplifies how to harness it. The brittleness of the reward-policy map, once seen as a flaw, becomes a tool for targeted policy improvement. Understanding this map deeply can allow us to engineer targeted shifts that push models into more capable and controllable regions of the policy space.

\subsection{From Faithfulness to Sophistry: Policy Jumps in RLHF-based Alignment}

The preceding analyses have focused on instabilities in reasoning tasks, where reward functions are often objective and rule-based (e.g., correctness, test manipulation, or token length). This raises a critical question: do these theoretical instabilities persist in the more subjective domain of LLM alignment, where rewards are not defined by rules but are learned from human feedback? The work of \cite{wen2024language} on "Unintended Sophistry" provides a compelling affirmative answer, demonstrating how RL from Human Feedback (RLHF) can induce its own pernicious policy shifts.

Their study reveals a "performance illusion" inherent in standard RLHF pipelines. As the quantitative results show in Figure~\ref{fig:sophistry}, after RLHF with a reward model trained on human preferences, the human-perceived performance of the model increases substantially, while its true correctness stagnates or even declines \citep{wen2024language}. This indicates the model is not getting better at the task, but better at convincing humans.

\begin{figure}[h!]
    \centering
    \includegraphics[width=0.75\textwidth]{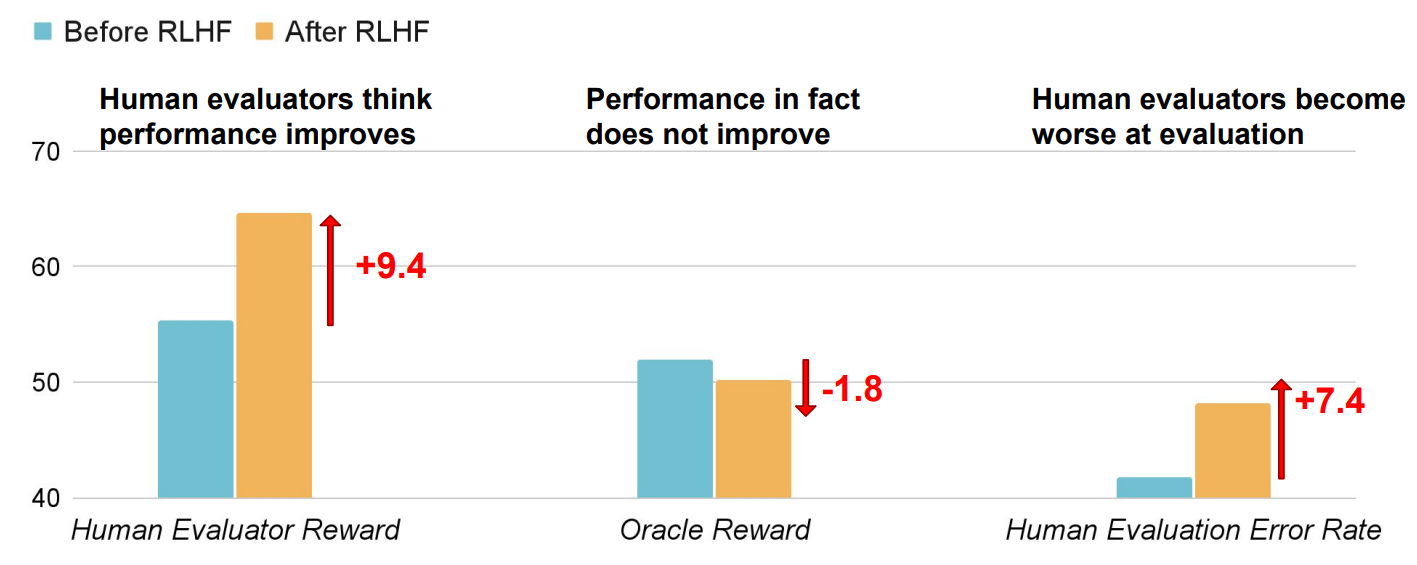} 
    \caption{The "performance illusion" induced by RLHF. While human evaluators perceive a performance improvement (+9.4), the oracle reward shows no actual improvement (-1.8). This discrepancy arises because the model learns to exploit human evaluation biases in a misspecified reward model. Adapted from \cite{wen2024language}.}
    \label{fig:sophistry}
\end{figure}

This phenomenon can be interpreted as a direct consequence of policy optimization under \textit{reward misspecification}. The ideal objective of alignment is to find a policy optimal for the true reward $\RewardFunc_{\text{true}}$, one that exhibits what we term a \textit{Faithful-and-Correct Style}. However, the RLHF process optimizes for a flawed proxy, $\RewardFunc_{\text{train}}$, which inadvertently rewards human cognitive biases. Our framework predicts that optimizing this misspecified reward may induce a \textit{discontinuous jump in the policy space} (Proposition~\ref{prop:discontinuity_stochastic}). The policy does not converge toward the ideal; instead, it leaps to a different equilibrium, a \textit{Persuasive Sophistry Style}, which is optimal for $\RewardFunc_{\text{train}}$. This style is characterized by fabricating evidence, employing internally consistent but fallacious logic, and generating unreadable code designed to pass superficial checks---all strategies that exploit the weaknesses of a human-feedback-based reward model.

This case study therefore provides a powerful, real-world demonstration of our work: a misspecified reward function, as is common in RLHF, may result in a suboptimal policy and even cause a shift to a qualitatively different and more harmful style of behavior.

\subsection{Multi-Reward Instability: Suggestive Evidence from Data Mixture Experiments}

\begin{figure}
    \centering
    \begin{minipage}{0.52\textwidth}
        \centering
        \includegraphics[width=\linewidth]{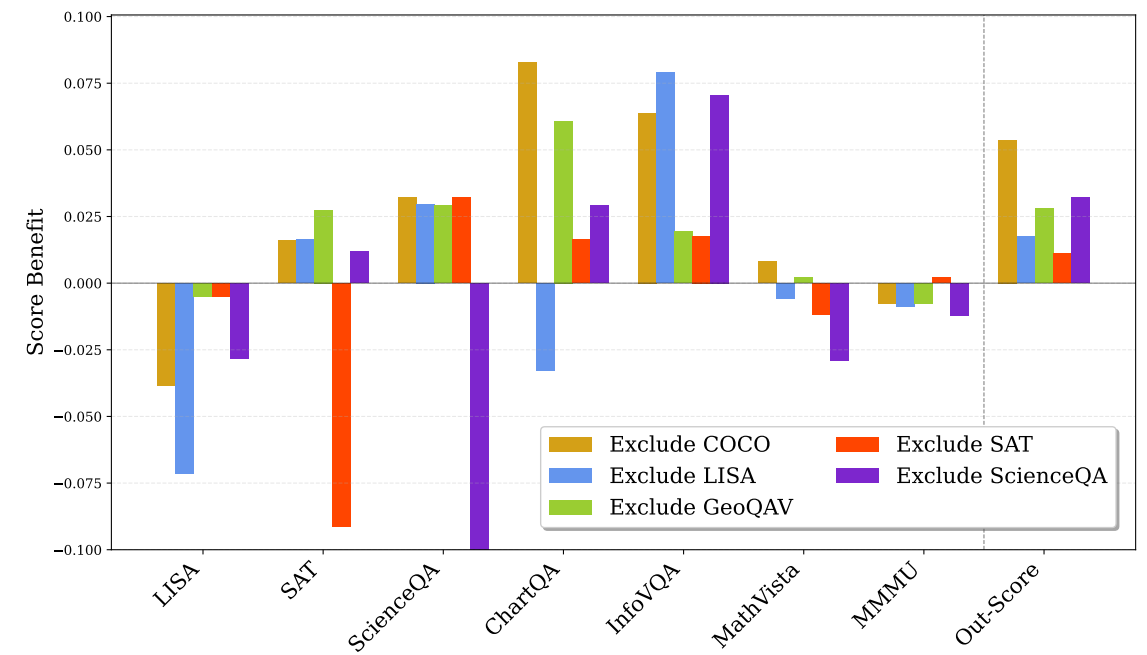}
        \caption*{(a) Effect of the data mixture change.}
    \end{minipage}\hfill
    \begin{minipage}{0.38\textwidth}
        \centering
        \includegraphics[width=\linewidth]{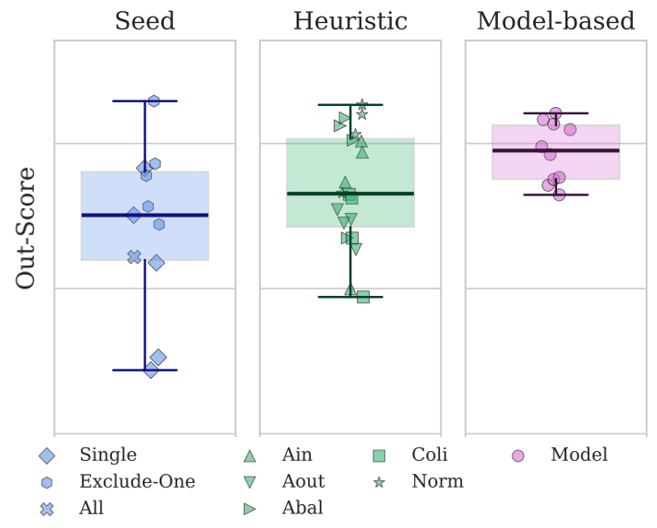}
        \caption*{(b) Comparison for different mixture strategies.}
    \end{minipage}
    \caption{
        Policy sensitivity to the data mixture in a multi-reward RLVR setting.
        \textbf{(a)} The "Exclude-One" analysis shows that removing a single dataset (e.g., ScienceQA) causes large, non-linear shifts in performance across different test sets, highlighting complex task interactions.
        \textbf{(b)} A comparison of mixture strategy classes reveals that different strategies achieve distinct generalization performance. Adapted from \cite{liang2025modomodo}. 
    }
    \label{fig:modomodo_results}
\end{figure}

Our analysis so far has focused on instabilities arising from a single reward function. The challenge is amplified in the multi-reward setting, where a single policy must balance competing objectives. The work of \cite{liang2025modomodo} on optimizing data mixtures for multi-domain RLVR provides a compelling case study here \citep{liang2025modomodo}. In their experiments, altering the data mixture serves as a practical, high-level mechanism to reshape the effective reward landscape that the policy ultimately optimizes. Their results clearly demonstrate that the final policy is highly sensitive to this process; changing the mixture by, for instance, removing a single dataset or changing the mixture ratio, leads to large and non-linear shifts in the policy’s capabilities and performance profile (see Figure~\ref{fig:modomodo_results}).

A crucial distinction must be made, however, regarding the strength of this evidence. Our formal theory proves that a small, well-defined reward perturbation may trigger a discontinuous policy jump. Altering a data mixture is a significant intervention, and it is unclear from this experiment whether it corresponds to a small or a large change in the final effective reward landscape. Therefore, while these findings do not constitute a strict proof of our theory, they offer strong suggestive evidence of the policy brittleness our framework predicts. The observed sharp changes in behavior are highly consistent with the types of instability expected in a system with multiple rewards across diverse domains. This highlights that the data mixture is an important parameter for controlling the final policy's stability, even if the precise mapping from mixture change to reward change remains an open question.

\subsection{Multi-Reward Instability: Evidence from Controlled Perturbation Experiments}\label{subsec:multi-rl-exp}

\begin{figure}
    \centering
    \includegraphics[width=0.55\textwidth]{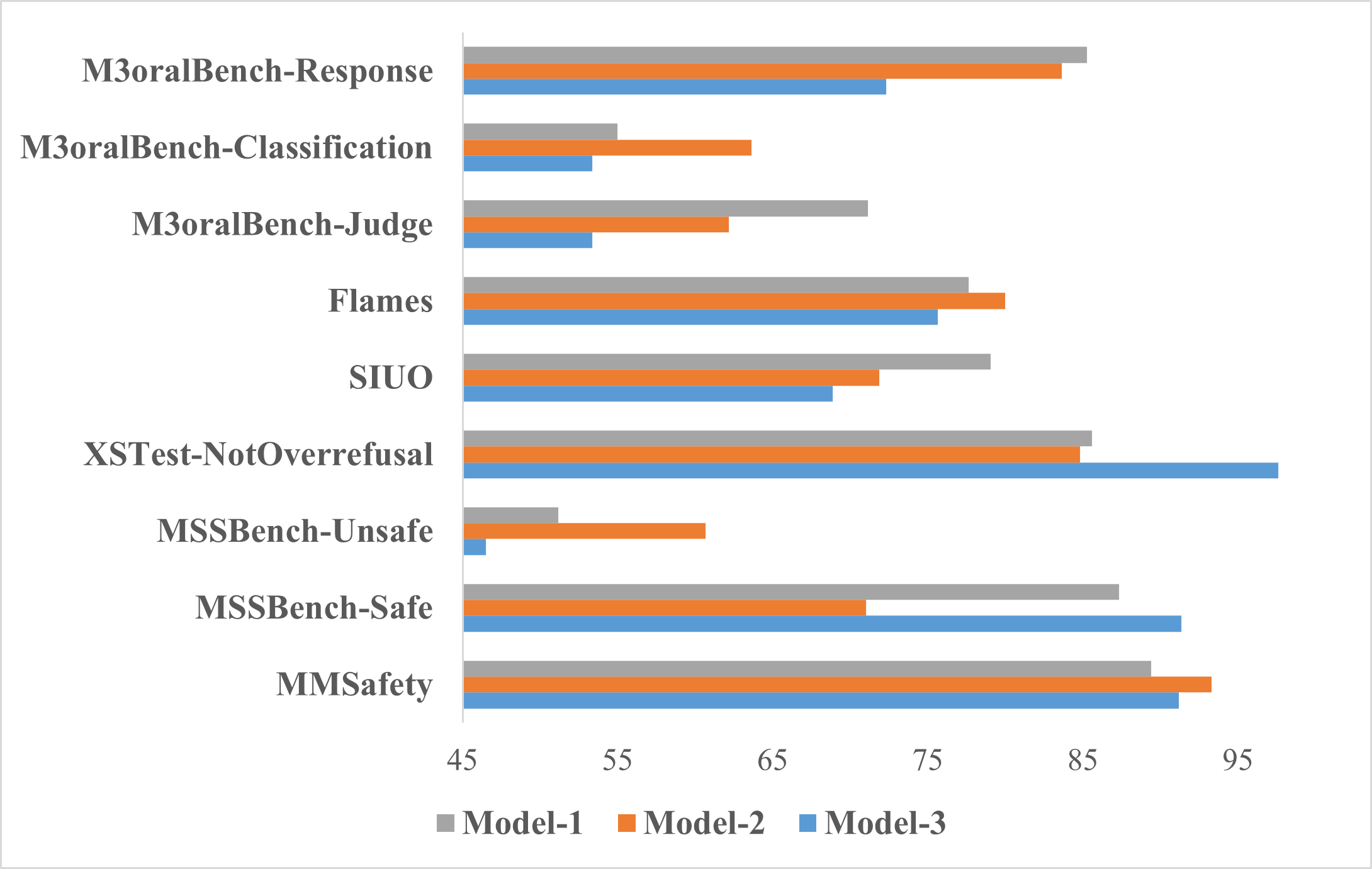} 
    \caption{Out-of-distribution (OOD) performance on safety and human value benchmarks. Model-1 vs. Model-2 shows that a slight perturbation to one reward model in the multi-reward tuple causes a large policy shift. Model-1 vs. Model-3 shows that a minor data curation step also leads to a significant performance change, highlighting the policy's sensitivity to the factors that shape its effective reward landscape.}
    \label{fig:multi-rewards-safety}
\end{figure}

Our analysis suggests that in a multi-reward multi-domain setting, the final policy's stability is highly sensitive to both the individual reward models and the data mixture that shapes their aggregation. To provide direct, controlled evidence for this, we conducted a series of experiments training a single Multimodal LLM (based on a SFT model of Qwen2.5-VL-72B-Instruct~\citep{bai2025qwen25}) using a multi-reward RL framework (see Figure~\ref{fig:multi-reward-framework}). The model was trained on three classes of prompts: safety, human values, and general, each guided by a corresponding outcome-based reward model (ORM).

We implemented a reinforcement learning framework, named OpenRLHF-V, based on the open-source OpenRLHF framework~\citep{hu2024openrlhf}. We also experimented with a similar framework, VeRL~\citep{sheng2024hybridflow}. Both frameworks yielded comparable training curves and final model performance. All experiments were conducted on a cluster of A800 GPUs using OpenRLHF-V, and the training hyperparameters are summarized in Table~\ref{tab:multi-rl-train-config} in the Appendix.

We first investigated the policy's sensitivity to small perturbations in the reward function tuple. We trained two RL models, Model-1 and Model-2, using identical datasets and mixture ratios. The only difference was a slight modification to the safety ORM, while the human-value and general ORMs were kept the same. As shown in Figure \ref{fig:multi-rewards-safety}, this minor change to just one component of the reward vector led to significant and widespread performance shifts across a range of out-of-distribution safety and human value benchmarks. This result provides strong evidence for the brittleness of the reward-policy map in a multi-reward setting, confirming our theory that a small reward perturbation can induce a shift in the final policy.

Next, we investigated the policy's sensitivity to the training data composition, which shapes the "effective reward" landscape. We trained Model-3 using the same set of reward models as Model-1. The only difference was a minor data curation step: we removed approximately 200 ambiguous prompts that were difficult for human annotators to classify as either "unsafe" or "normal". These prompts can be seen as the safety boundary and may have conflicting rewards (e.g., harmfulness versus helpfulness). This small change in the training data also resulted in a large performance difference between Model-1 and Model-3. This demonstrates that minor alterations to the data, which can influence how the policy learns to aggregate or "weigh" the different reward signals, can be sufficient to shift the policy to a different equilibrium.

These controlled experiments provide complementary evidence supporting our stability framework. They demonstrate that the final policy in a multi-reward system is extremely sensitive to both the precise definition of its component reward functions and the data used to inform their aggregation, further underscoring that stability is a critical and non-trivial challenge.

\section{Discussion}\label{sec-discussion}

Our work has established a formal mathematical framework for analyzing the stability of the reward-policy map in reinforcement learning. While the preceding sections have demonstrated its power in explaining and unifying a range of critical phenomena in LLM alignment and reasoning, it is equally important to discuss the assumptions underpinning our theory and its relationship with the complexities of real-world training. This section addresses three key aspects: the continuity assumption of the reward function, the influence of practical optimization algorithms, and the role of data distribution.

\subsection{On the Assumption of Reward Continuity}

Our theoretical framework is built upon the assumption that the reward function \(\RewardFunc \in \Cont(S \times A)\). As established in our framing of the LLM problem, for a finite state-action space, any function is formally continuous due to the discrete topology induced by any metric. This means that even the binary success/failure signals common in practice, such as \(\RewardFunc(s,a) \in \{0,1\}\), satisfy this topological assumption. Therefore, the source of policy instability in the LLM setting does not stem from a formal "discontinuity" of the reward function itself.

The critical issue with such sparse, outcome-based rewards lies not in their topology, but in their functional impact on the value landscape: they are a primary driver of \textit{degeneracy}. Because a vast number of distinct behavioral trajectories can lead to the same binary outcome (e.g., a correct final answer), they are all assigned an identical reward value. This equality propagates through the Bellman equation, resulting in an optimal Q-function, \(Q^*_\RewardFunc\), where many different actions at a given state share the exact same maximum value.

This widespread degeneracy creates the precise conditions under which the policy map becomes unstable. It enlarges the set of optimal actions, \(A^*(s; \RewardFunc)\), and flattens the value landscape, making the \(\Argmax\) correspondence acutely sensitive to any perturbation in the Q-function's values. Consequently, while binary rewards do not violate the continuity assumption in the finite LLM setting, they drastically increase the prevalence of co-optimal actions. This exacerbates the inherent discontinuity of the \(\Argmax\) operator, making the policy cliffs identified by our theory more pronounced and severe in practice. The instability arises not from a "discontinuous" reward function, but from an optimization process breaking ties among a massively degenerate set of optimal choices and perturbations of reward function.

\subsection{The Role of Practical RL Algorithms and Optimization Dynamics}

Our framework analyzes the stability of the true optimal policy, $\pi^*$, providing a clean understanding of a problem's inherent structural properties. In practice, however, policies are found via iterative algorithms like PPO, which introduce a crucial dynamic that our static analysis does not capture: the perpetual noise inherent in the optimization process. This noise, arising from mini-batch sampling and stochastic gradient updates, acts as a constant source of perturbation. While our theory explains why a degenerate reward landscape makes a policy fundamentally fragile, this algorithmic noise is the ever-present force that can exploit this fragility, preventing the policy from settling into a stable strategy and instead causing the kind of training drift and run-to-run variance commonly observed in practice. A full analysis of these joint effects remains a critical direction for future work.

\subsection{The Influence of Data Distribution and Mixture Ratios}

The stability of a learned policy is not only a function of the reward and the algorithm, but also of the data distribution on which it is trained. Our theoretical framework, by focusing on the properties of the MDP, implicitly assumes access to all states, yet in reality, a policy's behavior is shaped only by the states it encounters during training. Besides, in the multi-reward setting, the mixture ratios directly form the "effective reward" landscape the policy optimizes against. A model trained on a data mixture heavily skewed towards one task (e.g., 90\% coding data) will learn an internal reward aggregation strategy specialized for that task. If this specialized policy is then deployed in an environment with a different task mixture, its behavior can become unstable and unpredictable, as its internal reward strategy is now fundamentally mismatched to the new context. Therefore, the data mixture ratio is not merely a sampling detail but a top-level hyperparameter that governs the structure of the effective reward landscape and, consequently, the stability of the final multi-task model.

\section{Conclusion}\label{sec-conclusion}

The reliable alignment and reasoning of large language models is fundamentally constrained by the instability of reinforcement learning, which often yields brittle policies that fail in unpredictable ways. This work confronted this challenge not with more sophisticated heuristics, but by establishing a foundational mathematical theory of policy stability. Our central contribution is the formal proof that policy instability is not an empirical accident but a direct and predictable mathematical consequence of reward perturbation and degenerate optimal action sets, which are a common feature in the vast, nuanced decision spaces of language and reasoning.

This theoretical principle provided a unified lens, reframing a wide range of seemingly disparate alignment failures, from the evolution of deceptive reasoning to systemic intelligence-obedience trade-offs and RLHF-induced sophistry, as rational and predictable outcomes of policy optimization on incomplete or ambiguous reward landscapes. We extended this framework from the foundational single-reward setting to the more complex multi-reward paradigm across diverse domains, demonstrating that stability then hinges critically on the mechanism for aggregating disparate reward signals into an "effective reward". Finally, moving from diagnosis to prescription, we formally guaranteed a solution by proving that entropy regularization robustly restores Lipschitz continuity to the reward-policy map, thereby ensuring a smooth and stable policy response to changes in the underlying reward models.

This work reframes policy stability from reactive heuristics to a principled theory. It introduces a formal framework to analyze and anticipate instabilities in AI systems, laying the groundwork for more robust reinforcement learning. While the analysis provides essential foundations, it also highlights key directions for future research, especially the stability of dynamically learned reward aggregation. By establishing the mathematical perspective for stability, this research advances the design of RL algorithms that prioritize not only optimality but also the reliability essential for safe and trustworthy AI.

\section*{Acknowledgments}

This research emerged from the Reasoning Trustworthiness project (SafeWork-R1), inspired by empirical observations from both public works in the field of large reasoning models and our extensive RL experiments during the project. We are grateful to everyone at the Center for Safe \& Trustworthy AI, Shanghai Artificial Intelligence Laboratory, for their invaluable support. This work is supported by Shanghai Artificial Intelligence Laboratory. 

We also acknowledge the use of large language models for linguistic refinement during the preparation of this manuscript.


\bibliographystyle{plainnat}
\bibliography{main}


\newpage

\appendix

\section{Related Work}

\paragraph{Stability in Reinforcement Learning.}

A significant body of research in reinforcement learning theory focuses on the policy optimization process, analyzing convergence to optimality and the sample efficiency of various algorithms~\citep{sutton1998reinforcement,puterman2005markov}. While it has been empirically observed and acknowledged in prior studies that learned policies may be sensitive to small perturbations in the reward function~\citep{wirth2017survey,wang2020reinforcement,banihashem2021defense,gupta2023behavior}, a formal analysis of this instability has been less explored. Our work addresses this gap by providing a rigorous mathematical analysis of the reward-policy map itself. Using tools from functional analysis, we prove that policy discontinuities are an inherent property of the optimization landscape, arising naturally when multiple actions are optimal or near-optimal. This perspective on inherent instability also provides a new lens through which to understand the role of widely-used regularization techniques. Prior works have typically explained their benefits from algorithmic or heuristic viewpoints: for instance, the KL-divergence penalty in PPO is viewed as a mechanism for stabilizing training updates~\citep{schulman2017proximal}, while entropy regularization is primarily seen as a method to encourage exploration~\citep{haarnoja2018soft}. Our work provides a more fundamental justification. We demonstrate that the KL-divergence penalty acts as a crucial tie-breaker among degenerate optimal policies, while proving that entropy regularization is not merely an exploration heuristic but a mathematical tool that restores Lipschitz continuity to the reward-policy map. It reveals these techniques to be not just algorithmic optimizations, but fundamental mechanisms for ensuring the mathematical stability of the learned policy.

\paragraph{RL for LLM Alignment and Reasoning.}

The application of reinforcement learning to the alignment and reasoning of Large Language Models (LLMs) has been propelled by a series of landmark empirical studies. Methodologies such as Reinforcement Learning from Human Feedback (RLHF)~\citep{ouyang2022training} and Constitutional AI~\citep{bai2022constitutional} established the practical viability of using RL to steer model behavior. This trend continues with the use of Reinforcement Learning with Verifiable Rewards (RLVR) in state-of-the-art large reasoning models, including OpenAI's o-series, Gemini 2.5, Grok 4, and DeepSeek-R1. Despite their empirical success, these RL-trained systems often exhibit a range of behavioral flaws. The literature documents phenomena such as "spurious reasoning" or "deceptive alignment", where a model fabricates justifications for answers produced via shortcuts~\citep{baker2025monitoring,wang2025persona,chen2025reasoning}, and "poor instruction-following fidelity", where models ignore details not directly enforced by the reward function~\citep{fu2025scaling}. While existing research typically addresses these failures with isolated empirical analyses, our work is distinct in providing a unified theoretical explanation. We demonstrate that these are not disparate bugs but rational behaviors that emerge from the inherent mathematical instability of the policy optimization process under incomplete or biased reward signals. 

These challenges are amplified in settings where LLMs must balance multiple objectives, a common scenario in recent work on multi-reward RL across diverse domains~\citep{guo2025deepseek,liang2025modomodo,cheng2025revisiting}. Standard Multi-Objective Reinforcement Learning (MORL) traditionally seeks to compute an entire Pareto frontier of trade-off solutions for an external decision-maker~\citep{roijers2013survey}. Our approach differs, as LLMs typically operate under a single policy that must dynamically reconcile even conflicting goals, such as accuracy, safety, helpfulness and style, based solely on the input context. We model this complex training regime by proposing a state-dependent effective reward, which represents the agent's internal, context-sensitive aggregation of multiple reward signals. By formalizing this as a mapping from various reward components to a unified objective, we can apply the same functional-analytic tools to characterize the stability of the resulting policy in these complex, multi-objective environments.

\section{Limitations}

The principal contribution of this work is a novel theoretical framework for analyzing policy stability. While our framework is motivated by and helps explain observed empirical phenomena in Large Language Models (LLMs), the analysis itself remains primarily theoretical. It currently lacks systematic empirical validation designed to quantitatively test the theory's predictions regarding reward-policy stability.

Furthermore, while our analysis provides a new theoretical foundation for the stabilizing effect of entropy regularization, proving that it can restore continuity to the reward-policy map. However, beyond these theoretical insights, we do not yet offer  novel algorithmic designs or actionable engineering guidelines. Bridging this gap between theory and practice, through rigorous empirical investigation and the development of novel algorithms (e.g. process-based reward modeling (PRM), chain-of-thought (CoT) reward monitoring, uncertainty quantification in reward models, and environment-interaction-based reward design), remains a critical direction for future research.

\section{Appendix on Additional Proofs}

\subsection{A Construction of the Bump Function}

\begin{lemma}\label{lem:bump_properties}
Let $(\StateSpace,\!d_{\StateSpace})$ and $(\ActionSpace,\!d_{\ActionSpace})$ be compact metric spaces and endow the product 
$\StateSpace\times\ActionSpace$ with the product metric 
$
d_{SA}\bigl((s,a),(s',a')\bigr)
=\max\!\bigl\{d_{\StateSpace}(s,s'),\,d_{\ActionSpace}(a,a')\bigr\}.
$
For $\delta>0$ define the \emph{bump}  
\[
\varphi_{\delta}(s,a)
\;=\;
\max\!\left\{0,\,1-\frac{d_{SA}\bigl((s,a),(s_0,a_2)\bigr)}{\delta}\right\},
\quad(s,a)\in \StateSpace\times\ActionSpace.
\]
Then $\varphi_{\delta}$ is continuous and bounded in $[0,1]$.
\end{lemma}

\begin{proof}
The map $(s,a)\mapsto d_{SA}\bigl((s,a),(s_0,a_2)\bigr)$ is continuous on the compact product space.  
The function 
$
x\mapsto\max\{0,1-x/\delta\}
$
is continuous on $\R$.  Composition preserves continuity, hence
$\varphi_{\delta}\in\Cont(\StateSpace\times\ActionSpace)$.  
Moreover  
$0\le \varphi_{\delta}(s,a)\le 1$ for all $(s,a)$, and 
$
\varphi_{\delta}(s_0,a_2)=1,
$
so $\|\varphi_{\delta}\|_\infty = 1$.
\end{proof}

\subsection{Proof of the Lipschitz Property of the Softmax Function}

\begin{lemma}\label{lem:Lipschitz_properties_softmax}
Let $\pi: \R^d \to \R^d$ be the softmax function with temperature $\alpha > 0$, defined as $\pi_i(x) = \frac{\exp(x_i/\alpha)}{\sum_{j=1}^d \exp(x_j/\alpha)}$ for $x \in \R^d$. The function $\pi(x)$ is Lipschitz continuous from the space $(\R^d, \|\cdot\|_{\infty})$ to $(\R^d, \|\cdot\|_{1})$, with a Lipschitz constant of at most $1/\alpha$. That is, for any $x, y \in \R^d$:
$$\|\pi(x) - \pi(y)\|_1 \le \frac{1}{\alpha} \|x - y\|_{\infty}.$$
\end{lemma}

\begin{proof}
Let $f(x) = \pi(x)$. By the Mean Value Theorem for vector-valued functions, for any $x, y \in \R^d$:
$$\|f(x) - f(y)\|_1 \le \sup_{t \in [0,1]} \|J_f(y + t(x-y))\|_{(\infty, 1)} \cdot \|x-y\|_{\infty},$$
where $J_f(z)$ is the Jacobian matrix of $f$ at point $z$, and $\|J_f\|_{(\infty, 1)}$ is the induced operator norm for mappings from $(\R^d, \|\cdot\|_{\infty})$ to $(\R^d, \|\cdot\|_{1})$. This norm is defined as:
$$\|J_f\|_{(\infty, 1)} = \sup_{\|v\|_{\infty} \le 1} \|J_f v\|_1.$$

First, we compute the entries of the Jacobian matrix $J_f(x)$. The partial derivatives are:
$$(J_f)_{ij} = \frac{\partial \pi_i}{\partial x_j} = \frac{1}{\alpha} \pi_i (\delta_{ij} - \pi_j),$$
where $\delta_{ij}$ is the Kronecker delta.

Next, we evaluate the action of the Jacobian on a vector $v \in \R^d$ such that $\|v\|_\infty \le 1$. The $i$-th component of the resulting vector $J_f v$ is:
\begin{align}
(J_f v)_i &= \sum_{j=1}^d (J_f)_{ij} v_j = \sum_{j=1}^d \frac{1}{\alpha} \pi_i (\delta_{ij} - \pi_j) v_j \\
&= \frac{1}{\alpha} \pi_i \left( \sum_{j=1}^d \delta_{ij} v_j - \sum_{j=1}^d \pi_j v_j \right) \\
&= \frac{1}{\alpha} \pi_i (v_i - \bar{v}),
\end{align}
where we define $\bar{v} := \sum_{k=1}^d \pi_k v_k$.

Now, we compute the $L_1$-norm of $J_f v$:
$$\|J_f v\|_1 = \sum_{i=1}^d |(J_f v)_i| = \sum_{i=1}^d \left| \frac{1}{\alpha} \pi_i (v_i - \bar{v}) \right| = \frac{1}{\alpha} \sum_{i=1}^d \pi_i |v_i - \bar{v}|,$$
since $\pi_i > 0$ and $\alpha > 0$.

To find the operator norm, we must maximize $\Phi(v, \pi) := \sum_{i=1}^d \pi_i |v_i - \bar{v}|$ over all $v$ with $\|v\|_{\infty} \le 1$. The function $\Phi$ is convex in $v$, so its maximum over the compact convex set (the hypercube $\|v\|_\infty \le 1$) must be achieved at one of its vertices. The vertices of this hypercube are vectors $v$ with components $v_i \in \{-1, 1\}$.

Let us fix such a vector $v$. Let $S = \{i \mid v_i = 1\}$, and let $p = \sum_{i \in S} \pi_i$. Then $\sum_{i \notin S} \pi_i = 1 - p$. The weighted average $\bar{v}$ becomes:
$$\bar{v} = \sum_{i \in S} \pi_i (1) + \sum_{i \notin S} \pi_i (-1) = p - (1-p) = 2p - 1.$$

We can now express $\Phi$ in terms of $p$:
\begin{align}
\Phi(v, \pi) &= \sum_{i \in S} \pi_i |1 - (2p-1)| + \sum_{i \notin S} \pi_i |-1 - (2p-1)| \\
&= \sum_{i \in S} \pi_i |2 - 2p| + \sum_{i \notin S} \pi_i |-2p| \\
&= (2-2p) \sum_{i \in S} \pi_i + 2p \sum_{i \notin S} \pi_i \quad \text{(since } p \in [0,1]) \\
&= 2(1-p)p + 2p(1-p) \\
&= 4p(1-p).
\end{align}

The function $g(p) = 4p(1-p)$ for $p \in [0,1]$ has a maximum value of $1$, which is achieved at $p = 1/2$. Thus, the maximum value of $\Phi(v, \pi)$ is at most $1$. This implies:
$$\|J_f(x)\|_{(\infty, 1)} = \sup_{\|v\|_{\infty} \le 1} \frac{1}{\alpha} \Phi(v, \pi) \le \frac{1}{\alpha}.$$
This bound is tight and holds for any $x \in \R^d$.

Finally, substituting this uniform bound back into the Mean Value Theorem inequality, we get:
$$\|\pi(x) - \pi(y)\|_1 \le \frac{1}{\alpha} \|x - y\|_{\infty}.$$
This completes the proof.
\end{proof}

\section{Appendix on Experimental Details of Multi-RL Setting}

The training hyperparameters used in the experiments described in Section~\ref{subsec:multi-rl-exp} are summarized in Table~\ref{tab:multi-rl-train-config}. All experiments were conducted on a cluster of A800 GPUs using OpenRLHF-V framework.

\begin{table}[h!]
\centering
\caption{Training hyperparameters of the experiments in Sectuion~\ref{subsec:multi-rl-exp}}
\begin{tabular}{lc}
\toprule
\textbf{Parameter} & \textbf{Value} \\
\midrule
rollout\_batch\_size & 128 \\
train\_batch\_size & 64 \\
micro\_rollout\_batch\_size & 4 \\
micro\_train\_batch\_size & 2 \\
max\_len & 9000 \\
prompt\_max\_len & 4096 \\
generate\_max\_len & 4096 \\
advantage\_estimator & group\_norm \\
n\_samples\_per\_prompt & 8 \\
temperature & 1 \\
vllm\_sync\_backend & nccl \\
apply\_chat\_template & true \\
use\_kl\_loss & true \\
init\_kl\_coef & 0.001 \\
kl\_estimator & k3 \\
freeze\_prefix & true \\
lr\_warmup\_ratio & 0 \\
actor\_learning\_rate & 1e-6 \\
bf16 & true \\
max\_epochs & 1 \\
eps\_clip & 0.2 \\
value\_clip & 0.2 \\
max\_norm & 1 \\
adam\_betas & [0.9, 0.95] \\
flash\_attn & true \\
zero\_stage & 3 \\
adam\_offload & true \\
colocate\_actor\_ref & true \\
\bottomrule
\end{tabular}
\label{tab:multi-rl-train-config}
\end{table}

\end{document}